\documentclass[journal]{IEEEtran}

\usepackage{amsthm,amssymb,amsmath,amsfonts,algorithm,algorithmic,subfigure,graphicx,cite,color,flushend,mysymbol,mathrsfs,lipsum,epsfig,epstopdf,booktabs}
\usepackage{multirow}
\usepackage{enumitem}

\AtEndDocument{\par\leavevmode}

\begin{document}

\pagenumbering{gobble}

\title{Distributed State Estimation Using Intermittently Connected Robot Networks}

\author{Reza~Khodayi-mehr, Yiannis~Kantaros, and~Michael~M.~Zavlanos%
\thanks{This work is supported in part by the ONR under grant $\#N000141812374$. Reza Khodayi-mehr, Yiannis Kantaros, and Michael M. Zavlanos are with the Department of Mechanical Engineering and Materials Science, Duke University, Durham, NC 27708, USA, {\tt\footnotesize \{reza.khodayi.mehr, yiannis.kantaros, michael.zavlanos\}@duke.edu}.}
}

\maketitle
  
\begin{abstract}
This paper considers the problem of distributed state estimation using multi-robot systems. The robots have limited communication capabilities and, therefore, communicate their measurements intermittently only when they are physically close to each other. To decrease the distance that the robots need to travel only to communicate, we divide them into small teams that can communicate at different locations to share information and update their beliefs. Then, we propose a new distributed scheme that combines (i) communication schedules that ensure that the network is intermittently connected, and (ii) sampling-based motion planning for the robots in every team with the objective to collect optimal measurements and decide a location for those robots to communicate. To the best of our knowledge, this is the first distributed state estimation framework that relaxes all network connectivity assumptions, and controls intermittent communication events so that the estimation uncertainty is minimized. We present simulation results that demonstrate significant improvement in estimation accuracy compared to methods that maintain an end-to-end connected network for all time.
\end{abstract}
\IEEEpeerreviewmaketitle
   
\section{Introduction} \label{sec:Intro}

Distributed State Estimation (DSE) using mobile robot networks has a number of important applications, including robot localization \cite{PACML2000FBKT, DML2002RB}, SLAM \cite{SLAM2006DB, nerurkar2014c, DAIA2015ALDP, leung2012decentralized}, coverage \cite{ISDSCMSN2012JO}, target localization \cite{vander2015algorithms, OPPRAATL2015FMZ, freundlich2017distributed} and tracking \cite{OSPMCTT2006MB, hollinger2009efficient, DMCSDTT2006CBM, OACTTPG2009DSH, DTMSNIM2007O, huang2015bank} among others.
In these applications, the robots are equipped with sensing devices and collect information in order to minimize the uncertainty of the state. Successful accomplishment of this task critically depends on the ability of the robots to communicate and exchange information with each other. Existing literature on DSE  often assumes
that the communication network remains connected for all time or over time. All-time connectivity imposes proximity constraints on the robots so that exploring large-scale environments becomes very inefficient. On the other hand, approaches that allow connectivity to be lost always assume that it will eventually be regained, without providing any theoretical guarantees for this. In this work we propose a distributed control framework for state estimation tasks that allows the robots to temporarily disconnect in order to optimally collect information in their environment but also ensures  that this information can be shared intermittently with all other robots through appropriate communication events in which the robots participate.

Specifically, we consider networks of mobile robots that are tasked with estimating a collection of time varying hidden states. We assume that the robots have limited communication capabilities so that they communicate their measurements, only when they are sufficiently close to each other. To minimize the distance that the robots need to travel only to communicate, as in our previous work \cite{kantaros2017temporal}, we partition the robot network into small teams of robots so that the graph of teams is connected, where we define edges between teams that have robots in common. Then, we propose  a correct-by-construction, distributed, discrete controller that designs sequences of communication events for the robots in every team so that the overall robot network is intermittently connected infinitely often. In \cite{kantaros2017temporal}, we require that communication events take place when the robots in every team meet at a common location in space, selected from a finite set of available locations. Instead, here we select the location of the communication events in continuous space and depending on the current status of the estimation task.
During every communication event, the robots update their beliefs using an estimation filter and decide on the paths they will follow until they meet again as a team. The connected subnetworks associated with communication events and the paths that the robots in every team follow until they communicate are determined by an online, distributed, sampling-based motion planner that takes into account the robots' objective to gather information in their environment. To the best of our knowledge, this is the first DSE framework with \textit{intermittent communication control}. We present simulation simulation results that show significant improvement in estimation accuracy compared to methods that maintain connectivity for all time.


\subsection{Distributed State Estimation}
Multi-robot DSE has been thoroughly addressed in the literature. 
In much of this literature \cite{DML2002RB, nerurkar2014c, DAIA2015ALDP, ISDSCMSN2012JO, vander2015algorithms, OPPRAATL2015FMZ,freundlich2017distributed, freundlich2018distributed, OSPMCTT2006MB, DMCSDTT2006CBM, DTMSNIM2007O, DKFSN2007O}, the posterior distribution of the state is assumed Gaussian so that Kalman Filter (KF) equations can be employed to obtain it in closed-form.
In \cite{DKFSN2007O, IWC2012KFR, DTMSNIM2007O, ISDSCMSN2012JO, DRSSIA2012JASR}, consensus algorithms are integrated with the KF equations to allow the robots to agree on their local estimates of the state.
Under more general assumptions, such closed-form expressions do not exist and, instead,  Bayesian frameworks need to be used for state estimation \cite{PACML2000FBKT, SLAM2006DB, DDFCASN2004MD, DDF2016CA}.

 None of these papers consider communication constraints in the DSE problem. Such constraints can be in the form of limited communication ranges and rates, as well as delays or loss of data. The work in\cite{DADDSN2001ND} generalizes the KF equations to account for  delayed and out-of-order data in sensor networks. Following a different point of view, the authors in \cite{KFIO2004SSFPJS} show that there is a critical lower bound on the communication rate in the network below which the estimation error becomes asymptotically unbounded.
A localization method for networks that are not guaranteed to be fully connected is presented in \cite{leung2012decentralized}.
Common in all these works is that the communication network used to share the information is end-to-end connected for all time \cite{DRSSIA2012JASR} or intermittently connected \cite{leung2012decentralized}. To the contrary, here, we lift all connectivity assumptions and instead control the communication network so that it is intermittently connected infinitely often. 

An important aspect of the multi-robot DSE problem is the design of informative paths for the robots that allow them to explore the environment and collect measurements.
Typically, informative path planning relies on defining appropriate optimality indices on the posterior distribution of the state. For Gaussian distributions, such optimality indices are scalar functions of the posterior covariance, e.g., trace \cite{DTMSNIM2007O, ISDSCMSN2012JO}, determinant, or maximum eigenvalue \cite{freundlich2018distributed}. Entropy \cite{DAIA2015ALDP} and mutual information \cite{DRSSIA2012JASR} are some other common choices.
Distributed planning methods for DSE can be obtained depending on the filter that is used for state estimation and the optimality index. For example, the authors in \cite{SDCMRIGT2006MD} propose a scalable control scheme for information gathering when the optimal control objective is partially separable. Along the same lines, a distributed framework using a Bayesian filter is presented in \cite{DDFCASN2004MD}.
%
In \cite{DTMSNIM2007O, ISDSCMSN2012JO}, where the consensus KF is utilized for target tracking, the gradient of an appropriate potential function is used for planning that results in flocking of mobile robots around the targets while avoiding collision. Similarly, in \cite{DRSSIA2012JASR}, the agents follow the gradient of the expected mutual information to estimate the state of an environment.

In this work, we allow the robots to temporarily disconnect and accomplish their tasks free of communication constraints. Then, the path planning problem should provide both a sequence of informative measurement locations for DSE and locations of communication events used to share information among all robots.
 When the process that is estimated is time varying, delays during communication events should also be minimized to limit the time that robots wait without collecting any measurements. In order to address this challenging planning problem, we utilize sampling-based algorithms.
Such algorithms have been used in \cite{lan2016rapidly,hollinger2013sampling} for informative path planning. Specifically, a variation of the $\text{RRT}^{*}$ algorithm \cite{kuffner2000rrt,karaman2011sampling} is proposed in \cite{lan2016rapidly} that can design offline periodic trajectories to estimate the state of a dynamic field. The works in \cite{hollinger2013sampling,hollinger2014sampling} build upon the RRG algorithm \cite{karaman2011sampling} to design motion plans that maximize an information theoretic metric subject to budget constraints associated with the traveled distance. Nevertheless, this approach cannot be used when the length of the designed path, i.e., the budget, is not known \textit{a priori}, as in our work. Common in \cite{hollinger2013sampling,hollinger2014sampling} is that they address single-agent motion planning problems. Applying these algorithms to multi-agent problems requires exploration of the joint space of all robots, which makes them computationally intractable for large-scale networks. By dividing the group of robots into small teams and using sampling-based methods to jointly plan the motion only of the robots in each team, our method scales to much larger problems so that it can also be implemented in real-time.

\subsection{Network Connectivity}
Network connectivity has been widely studied recently as it is necessary for distributed control and estimation. Approaches to preserve network connectivity for all time typically rely on controlling the Fiedler value of the underlying graph either in a centralized \cite{OACTTPG2009DSH,kim2006maximizing,zavlanos2007potential} or distributed \cite{ji2007distributed,Zavlanos_IEEETRO08,yang2010decentralized,montijano2011adaptive,franceschelli2013decentralized,sabattini2013decentralized} fashion. A recent survey on graph theoretic methods for connectivity control can be found in \cite{Zavlanos_IEEE11}. However, due to the uncertainty in the wireless channel, it is often impossible to ensure all-time connectivity in practice. Moreover, all-time connectivity constraints may prevent the robots from moving freely in their environment to fulfill their tasks, and instead favor motions that maintain a reliable communication network.
Motivated by this fact, intermittent communication frameworks have recently been proposed \cite{FLSPPDF2016CF,hollinger2010multi,zavlanos2010synchronous,kantaros2016distributedInterm,kantaros2016simultaneous,kantaros2017temporal,kantaros2018distributed,guo2018gathering}. 
Specifically, \cite{FLSPPDF2016CF} addresses a data-ferrying problem by designing paths for a single robot that connect two static sensor nodes while optimizing motion and communication variables. More general intermittent connectivity frameworks for multi-robot systems are proposed in \cite{hollinger2010multi,zavlanos2010synchronous,kantaros2016distributedInterm,kantaros2016simultaneous,kantaros2017temporal,kantaros2018distributed,guo2018gathering}.
The work in \cite{hollinger2010multi} proposes a receding horizon framework for periodic connectivity that ensures recovery of connectivity within a given time horizon.
When connectivity is recovered in \cite{hollinger2010multi} the whole network needs to be connected. To the contrary, \cite{zavlanos2010synchronous,kantaros2016distributedInterm,kantaros2016simultaneous,kantaros2017temporal,kantaros2018distributed,guo2018gathering} do not require that the communication network is ever connected at once, but they ensure connectivity over time, infinitely often, as in the method proposed here. 
The key idea in \cite{zavlanos2010synchronous,kantaros2016distributedInterm,kantaros2016simultaneous,kantaros2017temporal,kantaros2018distributed} is to divide the robots into smaller teams and require that communication events take place when the robots in every team meet at a common location in space.
 While in disconnect mode, the robots can accomplish other tasks free of communication constraints. Assuming that the graph of teams, whose edges connect teams that have robots in common, is a connected bipartite graph, \cite{zavlanos2010synchronous} proposes a distributed control scheme that achieves periodic communication events, synchronously, at the meeting locations. Arbitrary connected team graphs are considered in \cite{kantaros2016distributedInterm,kantaros2016simultaneous,kantaros2017temporal,kantaros2018distributed} where discrete plans (schedules) are synthesized as sequences of communication events at the meeting locations. Integration of the intermittent communication framework \cite{kantaros2016simultaneous} with task planning for mobile robots that have to accomplish high-level complex tasks captured by Linear Temporal Logic (LTL) formulas  or arbitrary time-critical tasks is presented in \cite{kantaros2017temporal,kantaros2018distributed}, respectively.  
Compared to \cite{zavlanos2010synchronous,kantaros2016distributedInterm,kantaros2016simultaneous,kantaros2017temporal,kantaros2018distributed}, here we do not require that communication events take place when the robots in every team meet at a common location in space.
Instead, we require that communication events take place when the robots in every team form a connected subnetwork somewhere in the continuous space, which introduces challenging connectivity constraints in the proposed planning problem for every team. Distributed intermittent communication controllers for multi-robot data-gathering tasks is presented in \cite{guo2018gathering} that allows information to flow  intermittently only to the root/user. To the contrary, in our  proposed method, information can flow intermittently between any pair of robots and possibly a user in a multi-hop fashion.  

\subsection{Contributions}


The contributions of this paper can be summarized as follows. To the best of our knowledge, this is the first framework for DSE with \textit{intermittent communication control}. We show that the proposed sampling-based solution to the informative path planning problem is probabilistically complete and guarantees that the robots in every team form a connected subnetwork at the same time in order to communicate, minimizing in this way the time that the robots remain idle. The proposed framework scales with the size of the robot teams and not with the size of the network. As a result, our approach can be applied to large-scale networks. Moreover, we characterize the delay in propagating information across the network as a function of the structure of the robot teams. Finally, we show through comparative simulation studies that DSE using intermittent communication outperforms methods that maintain network connectivity for all time. We also present simulation studies that show the effect of the structure of the teams on the communication delays and on the estimation performance of our algorithm.

%
%

The rest of the paper is organized as follows. In Section \ref{sec:PF} the DSE problem using intermittently connected robot networks is presented. The proposed distributed control framework is presented in Section \ref{sec:HC} and its efficiency is verified through simulation studies in Section \ref{sec:Sim}. Conclusive remarks are presented in Section \ref{sec:Concl}. 

\section{Problem Definition} \label{sec:PF}


%
Let $\bbx(t)\in \reals^n$ denote a state variable that evolves according to the following nonlinear dynamics
\begin{equation} \label{eq:dynamics}
\bbx(t+1) =\bbf(\bbx(t), \bbu(t), \bbw(t) ),
\end{equation}
where  $\bbu(t) \in \reals^{d_u}$ and $\bbw(t) \in\reals^{d_w}$ denote the control input and the process noise at discrete time $t$. We assume that the process noise $\bbw(t)$ is normally distributed as $\bbw(t) \sim \ccalN(\bb0, \bbQ(t) )$, where $\bbQ(t) \in \mbS^{d_w}_{++}$ and $\mbS^{d_w}_{++} \subset \reals^{d_w \times d_w}$ is the set of symmetric positive-definite matrices. 

Consider also $N$ mobile robots tasked with collaboratively estimating the state $\bbx(t)$ of the dynamic process in \eqref{eq:dynamics}. Let $\Omega \subset \reals^d$ be the domain where the robots live and let $O\subset\Omega$ denote the set of obstacles.  We assume that the robots collect measurements of the state $\bbx(t)$ inside the obstacle-free subset $\Omega_{\text{free}}=\Omega \backslash O$ of this domain according to the following sensing model
\begin{equation} \label{eq:measModel}
\bby(t, \bbq) = \bbh(\bbx(t), \bbq, \bbv(t)),
\end{equation}
where $\bby \in \reals^m$ is the measurement vector at discrete time $t$ taken at location $\bbq \in \Omega_{\text{free}}$ by one robot sensor and $\bbv(t) \sim \ccalN(\bb0, \bbR(t) )$ is the white measurement noise with covariance $\bbR(t) \in \mbS^{d_v}_{++}$.

Moreover, we assume that the robots have limited communication capabilities and, therefore, they can communicate their measurements only when they are sufficiently close to each other. Specifically, every robot is able to communicate with another robot if it lies within a communication range $R \ll \texttt{diam}(\Omega)$, where $\texttt{diam}(\Omega)$ is the diameter of the domain $\Omega$. Without loss of generality we assume that the communication range is the same for all robots.

Since the communication range is much smaller than the size of the domain in which the robots operate, requiring that all robots are connected either for all time or intermittently can significantly interfere with the tasks that they need to accomplish, especially if these require them to travel to possibly remote locations in the domain. Therefore, we do not require that the robots ever form a connected network at once and, instead, we divide the group of robots into $M\geq 1$ teams, denoted by $\ccalT_i$, where $i\in\{1,2,\dots,M\}$, and require that every robot belongs to exactly two teams.\footnote{In this work, for simplicity, we require that every robot belongs to two teams. However, more complex team membership graphs $\ccalG_\ccalT$ can be considered, as in \cite{kantaros2017temporal}, where robots can belong to any number of teams.} 


Given the teams $\ccalT_i$, we also define the graph of teams $\ccalG_{\ccalT}=\{\ccalV_{\ccalT},\ccalE_{\ccalT}\}$ whose set of nodes $\ccalV_{\ccalT}$ is indexed by the teams $\ccalT_i$, i.e., $\ccalV_{\ccalT}=\{1,2,\dots,M\}$ and set of edges $\ccalE_{\ccalT}$ consists of links between nodes $i$ and $j$ if $\ccalT_i\cap\ccalT_j\neq\varnothing$, i.e., there exist a robot $r_{ij}$ that travels between teams $\ccalT_i$ and $\ccalT_j$.

\begin{figure}[t]
  \centering
  \subfigure[]{
    \label{fig:GT1}
     \includegraphics[width=0.46\linewidth]{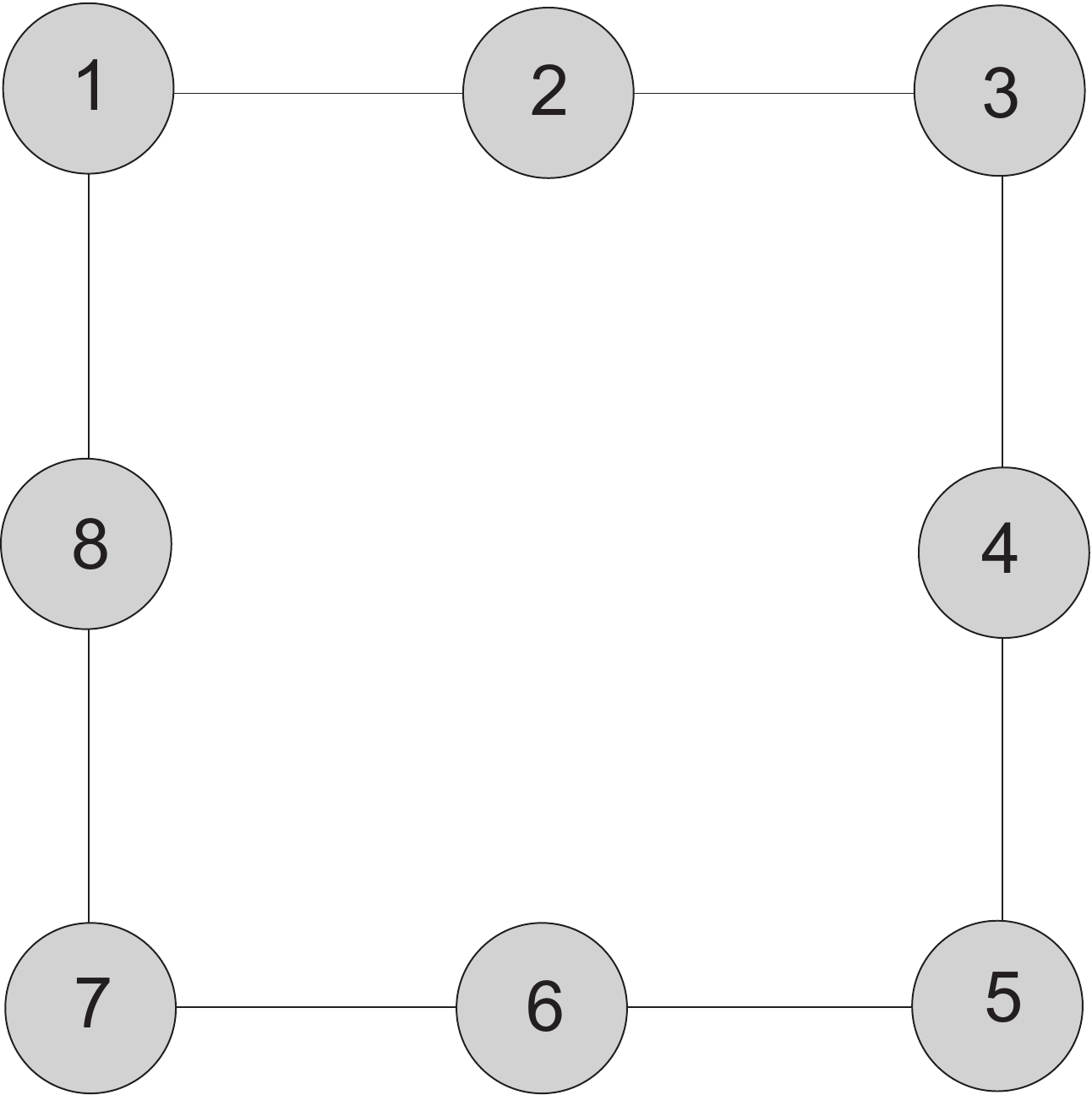}}
     \subfigure[]{
       \label{fig:GT2}
       \includegraphics[width=0.46\linewidth]{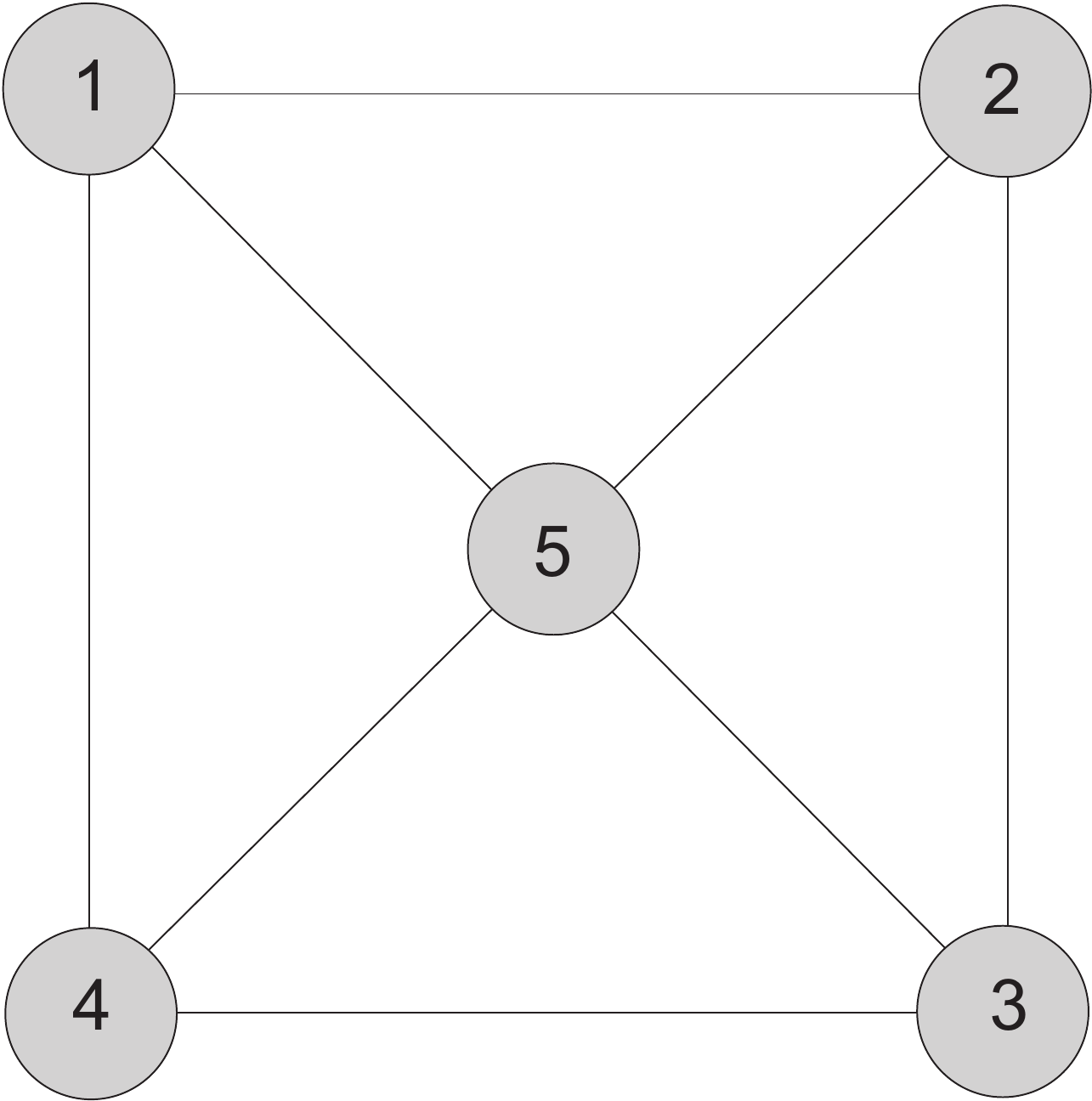}}
  \caption{Consider a network of N=8 robots. Figure \ref{fig:GT1} depicts the graph $\ccalG_{\ccalT}$ when the robots are divided into $M=8$ teams as: $\ccalT_1=\{r_{12},r_{18}\}$, $\ccalT_2=\{r_{12},r_{23}\}$, $\ccalT_3=\{r_{23}, r_{34}\}$, $\ccalT_4=\{r_{14},r_{45}\}$, $\ccalT_5=\{r_{56},r_{45}\}$, $\ccalT_6=\{r_{56},r_{67}\}$, $\ccalT_7=\{r_{67},r_{78}\}$, $\ccalT_8=\{r_{18},r_{78}\}$.  Figure \ref{fig:GT2} depicts the graph $\ccalG_{\ccalT}$ when the robots are divided into $M=5$ teams as: $\ccalT_1=\{r_{12}, r_{14}, r_{15} \}$, $\ccalT_2=\{r_{12}, r_{23}, r_{25} \}$, $\ccalT_3=\{r_{23}, r_{34}, r_{35} \}$ , $\ccalT_5=\{r_{15}. r_{25}, r_{35}, r_{45} \}$. }
  \label{fig:GT}
\end{figure}

We assume that the robots in every team $\ccalT_i$ communicate when they construct a connected network. Hereafter, we denote by $\ccalG_{\ccalT_i}=\{\ccalV_{\ccalT_i}, \ccalE_{\ccalT_i}(t)\}$ the communication graph constructed by robots in team $\ccalT_i$, where the set of nodes $\ccalV_{\ccalT_i}$ contains all robots in team $\ccalT_i$ and the set of edges $\ccalE_{\ccalT_i}(t)$ collects communication links that emerge between robots in $\ccalT_i$, whose pairwise distance is less than or equal to $R$. This way, a dynamic robot communication network is constructed as follows.

\begin{definition}[Communication Network $\ccalG_c(t)$]
The communication network among the robots is defined as a dynamic undirected graph $\ccalG_c(t)=\{\mathcal{V}_c,\mathcal{E}_c(t)\}$, where the set of nodes $\ccalV_c$ is indexed by the robots, i.e., $\ccalV_c=\{1,\dots,N\}$, and $\ccalE_c(t)\subseteq\ccalV_c\times\ccalV_c$ is the set of communication links that emerge among robots in every team $\ccalT_i$ when they form a connected graph $\ccalG_{\ccalT_i}$.
\end{definition}

To ensure that information is propagated among all robots in the network, we require that the communication graph $\ccalG_c(t)$ is \textit{connected over time infinitely often}, i.e., that all robots in every team $\ccalT_i$ form connected graphs $\ccalG_{\ccalT_i}$ infinitely often. For this, it is necessary to assume that the graph of teams $\ccalG_\ccalT$ is connected; see e.g., Figure \ref{fig:GT}. Specifically, if $\ccalG_\ccalT$ is connected, then information can be propagated intermittently across teams through robots that are common to these teams and, in this way, information can reach all robots in the network. 

We denote by $r_{ij}$ a robot that belongs to teams $\ccalT_i$ and $\ccalT_j$ and assume that it is governed by the following nonlinear dynamics
\begin{equation}\label{eq:rdynamics}
\bbp_{ij}(t+1)=\bbg(\bbp_{ij}(t),\bbu_{ij}(t)),
\end{equation}
where $\bbp_{ij}(t) \in \Omega_{\text{free}}$ stands for the position of robot $r_{ij}$ and $\bbu_{ij}(t) \in \reals^{d_r}$ stands for a control input. Without loss of generality, we assume that all robots have the same dynamics.

The problem that we address in this paper can be summarized as follows.
\begin{prob} \label{prob:intMon}
Given the dynamic process \eqref{eq:dynamics} and measurement model \eqref{eq:measModel}, and a network of $N\geq1$ robots divided into $M\geq1$ teams $\ccalT_i$, $i \in \set{1, \dots, M}$, such that $\ccalG_{\ccalT}$ is connected, determine paths $\bbp_{ij}(t) \in \Omega$ for all robots $r_{ij}$, so that
(i) the communication graph $\ccalG_c(t)$ is intermittently connected infinitely often, and
(ii) all robots collectively minimize estimation uncertainty of the state $\bbx(t)$ over time.
%
%
\end{prob}

Throughout the paper we make the following assumption.
\begin{assumption} [State Model]
The dynamics of the state \eqref{eq:dynamics}, the control vectors $\bbu(t)$ in \eqref{eq:dynamics}, the sensing model \eqref{eq:measModel}, and process and measurement noise covariances $\bbQ(t)$ and $\bbR(t)$ are known. Furthermore, the state dynamics \eqref{eq:dynamics} and sensing model \eqref{eq:measModel} are differentiable functions.
\end{assumption}
\section{Path Planning with Intermittent Communication} \label{sec:HC}
To solve Problem \ref{prob:intMon}, we propose a distributed control framework that concurrently plans robot trajectories that minimize a desired uncertainty metric, e.g., a scalar function of the covariance matrix \cite{freundlich2018distributed}, the Fisher information matrix \cite{meJ1}, or the entropy or mutual information of the posterior distribution of $\bbx(t)$ \cite{leahy2015distributed}, and schedules communication events during which the robots exchange their gathered information and update their beliefs. The schedules of communication events are determined by a correct-by-construction discrete controller that ensures that the communication network is intermittently connected; cf. Section \ref{sec:DP}. 
The connected subnetworks associated with these communications events and the paths the robots follow until they communicate are determined by an online sampling-based planner that takes into account the robots' objective to collect information; cf. Section \ref{sec:planning}.

\subsection{Distributed Intermittent Connectivity Control} \label{sec:DP}
In this section, we design infinite sequences of communication events (also called communication schedules) that ensure that robots in every team $\ccalT_i$ communicate intermittently and infinitely often, for all $i\in\{1,\dots,M\}$. Since every robot belongs to two teams, the objective in designing these schedules is that no teams that share common robots communicate at the same time, as this would require the shared robots to be present at more than one location at the same time. We call such schedules conflict-free. 

\begin{definition}[Schedule of communication events]\label{defn:sched}
The schedule of communication events of robot $r_{ij}$, denoted by $\texttt{sched}_{ij}$, is defined as an infinite repetition of the finite sequence 
\begin{align}\label{eq:si} 
 s_{ij}=&X,\dots,X,i,X,\dots,X,j,X,\dots,X,
\end{align} 
i.e., $\texttt{sched}_{ij}=s_{ij},s_{ij},\dots=s_{ij}^{\omega}$, where $\omega$ stands for the infinite repetition of $s_{ij}$.
\end{definition}

The detailed construction of such conflict-free schedules is omitted and can be found in \cite[Sec. V]{kantaros2017temporal}.  In Definition \ref{defn:sched}, $i$ and $j$ represent communication events for teams $\ccalT_i$ and $\ccalT_j$, respectively, and the discrete states $X$ indicate that there is no communication event for robot $r_{ij}$. The length of sequence $s_{ij}$ is $T=\max\left\{d_{i}\right\}_{i=1}^M+1$ for all robots $r_{ij}$, where $d_{i}$ denotes the degree of node $i\in\ccalV_{\ccalT}$. Also, we denote by $\texttt{sched}_{ij}(k_{ij})$ the $k_{ij}$-th entry in the sequence $\texttt{sched}_{ij}$, where $k_{ij}\in\mathbb{N}$. Hereafter, we call the indices $k_{ij}$ \textit{epochs}. The schedules in Definition \ref{defn:sched} ensure that the communication network is intermittently connected, i.e., all teams $\ccalT_i$ communicate infinitely often \cite{kantaros2017temporal}. 

Note that the schedules $\texttt{sched}_{ij}$ define the order in which robots $r_{ij}$ participate in communication events for the teams $\ccalT_i$ and $\ccalT_j$. Specifically, at an epoch $k_{ij}\in\mathbb{N}$, robot $r_{ij}$ needs to communicate with all robots that belong to either $\ccalT_i$ or $\ccalT_j$, if $\texttt{sched}_{ij}(k_{ij})=i$ or $\texttt{sched}_{ij}(k_{ij})=j$, respectively. If $\texttt{sched}_{ij}(k_{ij})=X$, then robot $r_{ij}$ does not need to participate in any communication event. By construction of the schedules $\texttt{sched}_{ij}$, it holds that the epochs when the communication events for a team $\ccalT_i$ will occur are the same for all $r_{ij}\in\ccalT_i$. Note also that due to the infinite repetition of the sequence $s_{ij}$, communication among robots in any team $\ccalT_i$ recurs every $T$ epochs. Specifically, denoting by $k_{\ccalT_i}^0$ the first epoch when communication happens within team $\ccalT_i$, communication within $\ccalT_i$ will also take place at epochs $\{k_{\ccalT_i}^0+nT\}_{n=0}^{\infty}$. 
For example, the schedule $\texttt{sched}_{ij}=[i,j,X]^{\omega}$ determines that robot $r_{ij}$ needs to communicate with robots in teams $\ccalT_i$ and $\ccalT_j$ at epochs $\{1+nT\}_{n=0}^{\infty}$ and $\{2+nT\}_{n=0}^{\infty}$, respectively. At epochs $\{3+nT\}_{n=0}^{\infty}$, robot $r_{ij}$ does not need to participate in any communication event. Finally, observe that the schedules $\texttt{sched}_{ij}$ do not determine the actual time instants or locations that these communication events should occur. These time instants and locations are determined by the planning problem; cf. Section \ref{sec:planning}.
%
     
Since the robots communicate intermittently, information is propagated across the network with a delay. In the following proposition, we show that this delay depends on the structure of the graph $\ccalG_{\ccalT}$ and on the period $T$ of the schedules $\texttt{sched}_{ij}$.

\begin{prop}\label{prop:kstar}
The worst case delay, measured in terms of elapsed epochs $k_{ij}$, with which information collected by robot $r_{ij}$ will propagate to every other robot in the network is $D_{\ccalG_\ccalT}=(T-1)L_{\ccalG_\ccalT}$, where $L_{\ccalG_\ccalT}$ is the longest shortest path in $\ccalG_{\ccalT}$ and $T$ is the period of the schedules $\texttt{sched}_{ij}$. Specifically, $L_{\ccalG_\ccalT}$ is defined as $L_{\ccalG_\ccalT}=\max_{i,e\in\{1,\dots,M\}}|\ell_{ie}|$, where $\ell_{ie}$ denotes the shortest path in $\ccalG_\ccalT$ that connects the teams $\ccalT_i$ and $\ccalT_e$, and $|\ell_{ie}|$ is the number of nodes in $\ell_{ie}$.
\end{prop}

\begin{proof}
Assume that robot $r_{ij}$ collects information at epoch $k_{ij}$. Without loss of generality, assume that the next communication event for robot $r_{ij}$ is with team $\ccalT_i$. By construction of the schedules,  at most $T-1$ epochs will elapse from epoch $k_{ij}$ until this communication event happens. Thus, the information collected by robot $r_{ij}$ at epoch $k_{ij}$ will be transmitted to all robots in team $\ccalT_i$ in $K_1\leq T-1$ epochs.
Next consider the number of epochs required for this information to be transmitted from team $\ccalT_i$ to any other team $\ccalT_e$. Since the graph $\ccalG_\ccalT$ is connected by assumption, there is at least one path that connects team $\ccalT_i$ to team $\ccalT_e$ and, therefore, information can be propagated from $\ccalT_i$ to $\ccalT_e$.
Let $\ell_{ie}$ denote the shortest path between teams $\ccalT_i$ and $\ccalT_e$. Then, this information will be transmitted from $\ccalT_i$ to $\ccalT_e$ through the path $\ell_{ie}$, within $(|\ell_{ie}|-1)(T-1)$ epochs, where $\abs{\ell_{ie} }$ is the length of the shortest path. Therefore, we get that $K_{ie}\leq(|\ell_{ie}|-1)(T-1)$. 
%
%
%
Finally, given an arbitrary team $\ccalT_i$, we get $K_{ie}\leq (L_{\ccalG_\ccalT}-1)(T-1)$, where $L_{\ccalG_\ccalT}$ stands for the longest shortest path in $\ccalG_\ccalT$, i.e., $L_{\ccalG_\ccalT}=\max_{i,e\in\{1,\dots,M\}}|\ell_{ie}|$ by definition. Therefore, the packet of information collected by robot $r_{ij}$ at epoch $k_{ij}$ will be transmitted to any other team $\ccalT_e$ and, consequently, to all robots within $D_{\ccalG_\ccalT} = K_1+K_{ie}\leq T-1 + (L_{\ccalG_\ccalT}-1)(T-1)=L_{\ccalG_\ccalT}(T-1)$ epochs.
\end{proof}


\begin{rem}[Discrete States $X$]\label{rem:X}
In the schedules $\texttt{sched}_{ij}$, defined in Definition \ref{defn:sched}, the states $X$ indicate that no communication events occur for robot $r_{ij}$ at the current epoch. These states are used to synchronize the communication events over the epochs $k_{ij}\in\mathbb{N}_+$, i.e., to ensure that the epochs $k_{ij}$ when communication happens for team $\ccalT_i$, $i\in\ccalM$, is the same for all robots $r_{ij}\in\ccalT_i$. Nevertheless, as shown in Theorem 7.8 in  \cite{kantaros2017temporal}, it is the order of the communication events in $\texttt{sched}_{ij}$ that is critical to ensure intermittent communication, and not the epochs or the time instants that they take place.
\end{rem}

\subsection{Informative Path Planning} \label{sec:planning}
%

The schedules $\texttt{sched}_{ij}$ developed in Section \ref{sec:DP} determine an abstract sequence of communication events that are not associated with any physical location in space. Similarly, the notion of an epoch indicates the place of a communication event within the sequence $\texttt{sched}_{ij}$ but is not associated with physical time. In this section, we embed the sequence of communication events $\texttt{sched}_{ij}$ over the epochs $k_{ij}$ into time $t$ and space $\Omega$. For this we design robot trajectories that not only allow the robots to obtain measurements of the state that  minimize estimation uncertainty, but also ensure that all robots in every team 
become connected and communicate with each other at the epochs specified by the schedules $\texttt{sched}_{ij}$. 

Let $k$ be an epoch when a communication event takes place for team $\ccalT_i$, i.e., $\texttt{sched}_{ij}(k)=i$, where to simplify notation we drop dependence of the epoch $k_{ij}$ on the robot $r_{ij}$. Moreover, assume that the path $\bbp_{ij}(t)$ is divided into segments indexed by the epochs and, let $\bbp_{ij}^k:[t_{0,ij}^{k}, t_{f,ij}^{k}  ] \rightarrow \Omega_{\text{free}}$ denote the $k$-th segment of path $\bbp_{ij}(t)$ of robot $r_{ij} \in\ccalT_i$ starting at the discrete time $t_{0,ij}^{k}$ and ending at $t_{f,ij}^{k} $. 
Communication within team $\ccalT_i$ during the $k$-th epoch takes place at the time instant $t_{f,\ccalT_i}^k = \max_{r_{ij}\in\ccalT_i} \set{ t_{f,ij}^{k} }$ when all robots $r_{ij}\in\ccalT_i$ arrive at the end locations $ \bbp_{ij}( t_{f,ij}^{k} ) $ of their paths $\bbp_{ij}^k$. The starting location $\bbp_{ij}(t_{0,ij}^k)$ of the path $\bbp_{ij}^k$ coincides with the location where the last communication event for robot $r_{ij}$ took place within team $\ccalT_j$; see also Figure \ref{fig:notations}. The epoch and corresponding time instant when communication within team $\ccalT_j$ took place are denoted by $\bbark$ and $t_{f,\ccalT_j}^{\bbark}$, respectively, where $k-T+1 \leq \bbark < k$ by periodicity of the communication schedules. Thus, robot $r_{ij}$ starts executing the path segment $\bbp_{ij}^k$ at the time instant $t_{0,ij}^k=t_{f,\ccalT_j}^\bbark$ and finishes its execution at $t_{f,ij}^{k} $.  
At time $t_{f,\ccalT_i}^k$, when all robots in $\ccalT_i$ communicate, they collectively design the path segments $\bbp_{ij}^{k+T}, \, \forall r_{ij} \in \ccalT_i$, that will result in a connected configuration for team $\ccalT_i$ at epoch $k+T$; see Figure \ref{fig:notations}.
In what follows, our goal is to design path segments $\bbp_{ij}^{k+T}$ that minimize estimation uncertainty and satisfy the following three constraints:

\begin{enumerate}[label=(\alph*)]
\item First, the paths $\bbp_{ij}^{k+T}$ do not intersect with the obstacles and respect the dynamics \eqref{eq:rdynamics}. 
\item Second, the end locations $\bbp_{ij}(t_{f,ij}^{k+T})$ for all robots $r_{ij}\in\ccalT_i$ correspond to a connected communication graph $\ccalG_{\ccalT_i}$ for team $\ccalT_i$.
\item Third, the end times $t_{f,ij}^{k+T}$ of the paths $\bbp_{ij}^{k+T}$ are the same and equal to $t_{f,\ccalT_i}^{k+T}$ for all robots $r_{ij}\in\ccalT_i$, so that there are no robots waiting for the arrival of other robots to communicate. Minimizing waiting times allows the robots to spend more time collecting measurements needed for estimation. 
\end{enumerate}
%

\begin{figure}[t]
  \centering
    \label{fig:notations1}
     \includegraphics[width=0.9\linewidth]{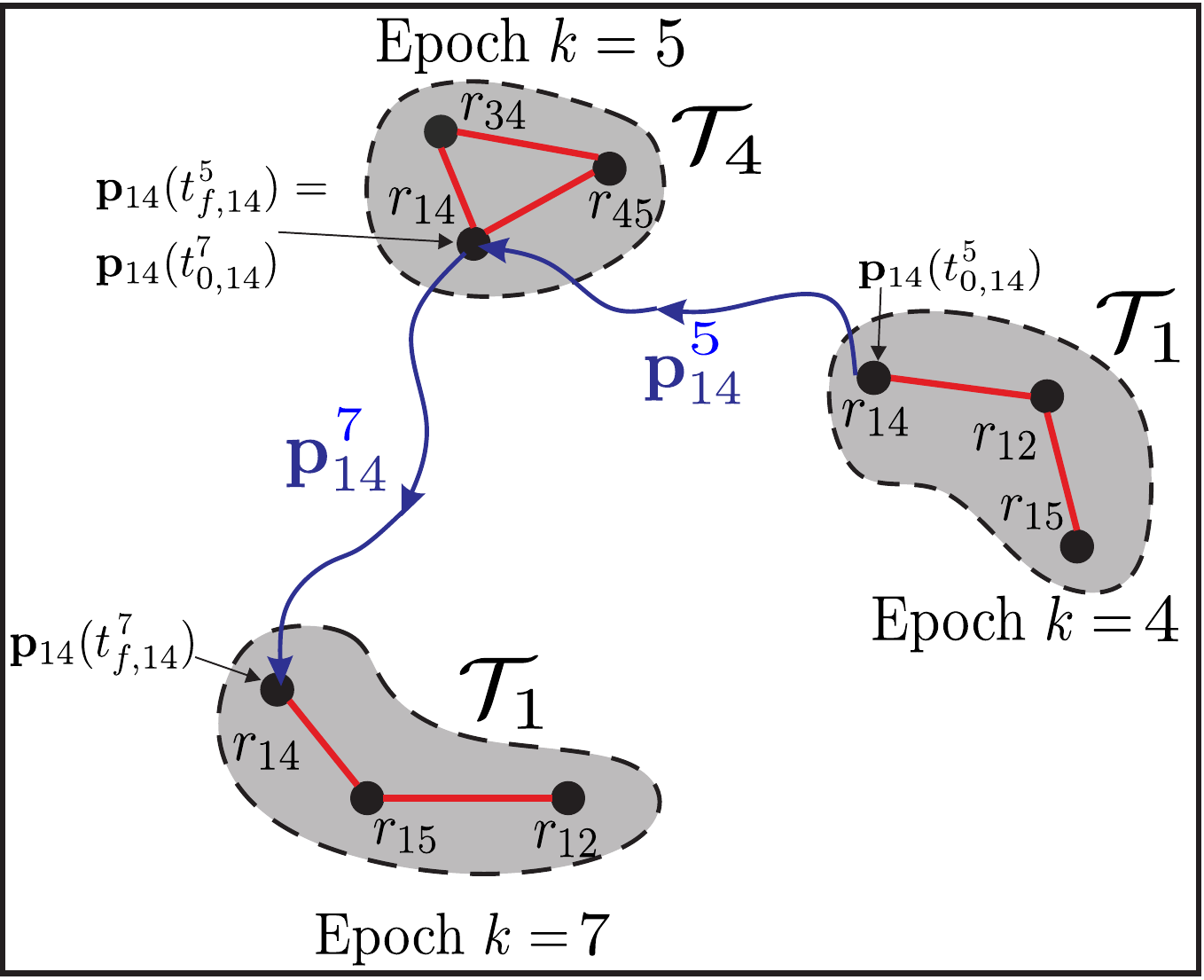}
  \caption{Illustration of the proposed planning algorithm for robot $r_{14}$. The communication schedule $\texttt{sched}_{14}$ has the following form, $\texttt{sched}_{14} =[1,~4,~X]^{\omega}$, with period $T=3$. The figure shows the path segments $\bbp^k_{14}$ that robot $r_{14}$ follows, where the paths $\bbp_{14}^5$ and $\bbp_{14}^7$ were designed at epochs $k=5-T=2$ and $k=7-T=4$, respectively. The starting location of the path $\bbp_{14}^{7}$ that leads to a connected configuration within team $\ccalT_1$ starts at the location $\bbp_{14}(t_{f,14}^5)$ at which the last communication event for robot $r_{14}$ occurred. Also, observe that since $\texttt{sched}_{14}(6)=X$, the path $\bbp_{14}^{6} = \emptyset$.}
  \label{fig:notations}
\end{figure}

To achieve this goal, we formulate an optimal control problem that the robots $r_{ij}\in\ccalT_i$ need to solve in order to design paths $\bbp_{ij}^{k+T}$ when they communicate at epoch $k$. 
Specifically, let $\ccalG^{k+T}_{\ccalT_i} = (\ccalV_{\ccalT_i}, \ccalE_{\ccalT_i}^{k+T})$ denote the communication graph of team $\ccalT_i$ at epoch $k+T$, where $\ccalV_{\ccalT_i} = \set{ r_{ij} \in \ccalT_i }$ and $ \ccalE_{\ccalT_i}^{k+T} = \{(r_{ij},r_{ie}) \, | \, \lVert\bbp_{ij}(t^{k+T}_{f,ij}) - \bbp_{ie}(t^{k+T}_{f,ie})\rVert \leq R \}$, where $R$ is the communication range.
We define by $\ccalC = \set{\ccalG (\ccalV, \ccalE) \, | \, \lambda_2(\ccalL(\ccalG)) > 0}$, the set of all connected graphs $\ccalG$ with vertices in the set $\ccalV$ and edges in the set $\ccalE$, i.e., the set of graphs $\ccalG$ whose Laplacian matrix $\ccalL(\ccalG)$ has positive second smallest eigenvalue \cite{Godsil_SPRINGER01}.

%
%
%
Then, the optimal control problem that we formulate to solve Problem \ref{prob:intMon} can be defined as follows:
\begin{align} \label{eq:optimProb}
& \min_{\substack{\bbP_{\ccalT_i}, t_{f,\ccalT_i} }} \sum_{t=t_{0,\ccalT_i}}^{t_{f,\ccalT_i}} \texttt{unc}(\bbP_{\ccalT_i}(t))  \\
& \ \st \bbp_{ij}(t^{k+T}_{0,ij}) = \bbp_{ij}(t^{\bbark}_{f,\ccalT_j}), \, \forall r_{ij} \in \ccalT_i \nonumber \\
& \ \ \ \ \ \ \bbp_{ij}(t) \in \Omega_{\text{free}}, \, \forall r_{ij} \in \ccalT_i \nonumber \\
& \ \ \ \ \ \  \bbp_{ij}(t+1) = \bbg(\bbp_{ij}(t),\bbu_{ij}(t)),~\forall r_{ij} \in \ccalT_i \nonumber \\
& \ \ \ \ \ \ \ccalG^{k+T}_{\ccalT_i} \in \ccalC,\nonumber\\
& \ \ \ \ \ \  t_{f,\ccalT_i}\geq t^{k+T}_{0,ij},  \, \forall r_{ij} \in \ccalT_i \nonumber\\
& \ \ \ \ \ \  t_{f,ij}=t_{f,\ccalT_i}, ~ \forall r_{ij} \in \ccalT_i,\nonumber\\
& \ \ \ \ \ \  \texttt{unc}(\bbP_{\ccalT_i}(t_{f,\ccalT_i}))\leq\delta . \nonumber
\end{align}
%
In \eqref{eq:optimProb}, $\bbP_{\ccalT_i}$ stands for the joint path of the robots $r_{ij}\in\ccalT_i$ that lives in the joint space $\Omega_{\text{free}}^{|\ccalT_i|}$. Projection of this path on the workspace of robot $r_{ij}$ yields the path segment $\bbp_{ij}^{k+T}$.
Also, $\texttt{unc}(\bbP_{\ccalT_i}(t)))$ denotes an arbitrary uncertainty metric such as a scalar function of the covariance matrix \cite{freundlich2018distributed}, the Fisher information matrix \cite{meJ1}, or the entropy or mutual information of the posterior distribution of $\bbx(t)$ \cite{leahy2015distributed}.
\footnote{The uncertainty metric $\texttt{unc}(\cdot)$ depends on the hidden state $\bbx(t)$ which in turn depends on the collected measurements and thus the robot paths $\bbP_{\ccalT_i}(t)$. In \eqref{eq:optimProb}, we consider the explicit dependence on $\bbP_{\ccalT_i}(t)$ for simplicity.}
Specifically, in Section \ref{sec:Sim}, we employ the  maximum eigenvalue of the posterior covariance as the uncertainty metric.
Then, the objective function measures the cumulative uncertainty in the estimation of the state $\bbx(t)$ after fusing information from all robots $r_{ij}\in\ccalT_i$ collected along their individual paths $\bbp_{ij}$ from $t_{0,\ccalT_i}=\min\{t_{0,ij}^{k+T}\}$ up to time $t_{f,\ccalT_i}$, given all earlier measurements available to team $\ccalT_i$. 
The first constraint in \eqref{eq:optimProb} enforces that the paths start from the location at which the previous communication event for the robots $r_{ij} \in \ccalT_i$ occurred.
The second constraint in \eqref{eq:optimProb} ensures that the designed paths are continuous and lie in the free space. The third constraint ensures that the dynamics \eqref{eq:rdynamics} is satisfied as the robots $r_{ij}\in\ccalT_i$ travel along their path segments  $\bbp_{ij}^{k+T}$. The fourth constraint ensures that the communication graph $\ccalG^{k+T}_{\ccalT_i}$ constructed at epoch $k+T$, once all robots in $\ccalT_i$ reach the end point of their respective paths $\bbp^{k+T}_{ij}$, is connected. The fifth constraint requires the final time $t_{f,\ccalT_i}$ to be greater than the initial time $t_{0,ij}^{k+T}$, for all robots $r_{ij}\in\ccalT_i$. The next constraint requires that all robots $r_{ij}\in\ccalT_i$ terminate the execution of their path segments $\bbp_{ij}^{k+T}$ at time instants $t_{f,ij}$ that are the same and equal to $t_{f,\ccalT_i}$, i.e., no waiting time. The final constraint in  \eqref{eq:optimProb} requires that the uncertainty metric $\texttt{unc}(\bbP_{\ccalT_i}(t_{f,\ccalT_i})))$ at time $t_{f,\ccalT_i}$ is below a specified threshold $\delta>0$. This ensures that the robots explore the environment sufficiently before they communicate again. This constraint can be replaced with other equivalent constraints serving the same purpose, e.g., minimum travel distance before communication. Appropriate selection of the constraint and the threshold value is problem-dependent; see Section \ref{sec:Sim} for more details. 


\subsection{Sampling-based Solution to Planning Problem} \label{sec:sampling}


Solving the path planning problem \eqref{eq:optimProb} in practice can be very challenging particularly because the function $\texttt{unc}(\bbP_{\ccalT_i}(t)) $ can be non-smooth, the set $\Omega_{\text{free}}$ is often non-convex, and the problem explicitly involves the time variable $t$.
Sampling-based algorithms can address the first two difficulties if the objective function is monotonic and bounded \cite{karaman2011sampling}.
The presence of the meeting time and locations in the optimal control problem \eqref{eq:optimProb} create additional challenges that further justify application of a sampling-based algorithm to solve \eqref{eq:optimProb}.
Specifically, here we propose a sampling-based algorithm that is built upon the RRT$^*$ algorithm \cite{karaman2011sampling}. The proposed algorithm is summarized in Algorithm \ref{alg:RRT}.

\begin{algorithm}[t]
\caption{Sampling-based Informative Path Planning}
\label{alg:RRT}
\begin{algorithmic}[1]
\STATE Set $\ccalV_s = \set{\bbv_0}$, $\ccalE_s = \emptyset$, and $\ccalX_g = \emptyset$;\label{rrt:line1}
\FOR{ $s = 1, \dots, n_{\text{sample}}$}\label{rrt:line2}
	\STATE Sample $\Omega_{\text{free}}$ to acquire $\bbv_{\text{rand}}$;\label{rrt:line3}
	\STATE Find the nearest node $\bbv_{\text{nearest}} \in \ccalV_s$ to $\bbv_{\text{rand}}$;\label{rrt:line4}
	\STATE  Steer from $\bbv_{\text{nearest}}$ toward $\bbv_{\text{rand}}$ to select $\bbv_{\text{new}}$;\label{rrt:line5}
		\IF { $\texttt{CollisionFree} (\bbv_{\text{nearest}}, \bbv_{\text{new}} )$ } 	\label{line:collisionFree}
			\STATE Update the set of vertices $\ccalV_s = \ccalV_s \cup \set{ \bbv_{\text{new}} }$; \label{rrt:line6}
			\STATE Build the set $\ccalV_{\text{near}} = \set{ \bbv \in \ccalV_s \, | \, \norm{\bbv - \bbv_{\text{new}}} < r }$;\label{rrt:line7}
			\STATE Extend the tree towards  $\bbv_{\text{new}}$ (Algorithm \ref{alg:extend});\label{rrt:extend}
			\STATE Rewire the tree (Algorithm \ref{alg:rewire});\label{rrt:rewire}
		\ENDIF
\ENDFOR
\STATE Find $\bbv_{\text{end}} \in \ccalX_g^i$ with smallest uncertainty; \label{rrt:node}
\STATE Return the path $\bbP_{\ccalT_i}^{k+T} = (\bbv_0, \dots, \bbv_{\text{end}} )$;\label{rrt:jointP}
\STATE Project $\bbP_{\ccalT_i}^{k+T}$ onto the workspace of $r_{ij}$ to get $\bbp_{ij}^{k+T}$;\label{projP}
\end{algorithmic}
\end{algorithm}   

Application of Algorithm \ref{alg:RRT} requires definition of a cost function and a goal set $\ccalX_g^i$. Referring to \eqref{eq:optimProb}, the cost of the joint path $\bbP_{\ccalT_i}$ that lives in $\Omega_{\text{free}}^{\abs{\ccalT_i}}$ is defined as $\texttt{Cost}(\bbP_{\ccalT_i})=\sum_{t=t_{0,i}}^{t_{f,\ccalT_i}} \texttt{unc}(\bbP_{\ccalT_i}(t))$. This function is additive and since the uncertainty metric $\texttt{unc}(\bbP_{\ccalT_i}) \geq 0$, it is also monotone,
 i.e., $\texttt{Cost}(\bbP_{\ccalT_i}^1)\leq\texttt{Cost}(\bbP_{\ccalT_i}^1|\bbP_{\ccalT_i}^2) = \texttt{Cost}(\bbP_{\ccalT_i}^1)+ \texttt{Cost}(\bbP_{\ccalT_i}^2)$, where $|$ stands for the concatenation of the paths $\bbP_{\ccalT_i}^1$ and $\bbP_{\ccalT_i}^2$. Thus, we can use a sampling-based algorithm to minimize it.
On the other hand, the goal set captures the constraints of the optimal control problem \eqref{eq:optimProb}. We define this set for team $\ccalT_i$ as
%
\begin{align}\label{goalset}
\ccalX_g^i  =\{\bbv \in \Omega_{\text{free}}^{|\ccalT_i|} \, |~&\text{(i)} \ \lambda_2( \ccalL( \ccalG^{k+T}_{ \ccalT_i} (\bbv)) ) > 0,  \nonumber\\& \text{(ii)} \min_{r_{ij}\in\ccalT_i} t_{ij}(\bbv)=\max_{r_{ij}\in\ccalT_i} t_{ij}(\bbv)  \nonumber\\& \text{(iii)}~\texttt{unc}(\bbv)\leq\delta\}.
\end{align}
In words, the goal set $\ccalX_g^i$ collects all points $\bbv \in \Omega_{\text{free}}^{|\ccalT_i|}$ that satisfy three conditions. First, the configuration $\bbv$ should correspond to a connected communication graph at epoch $k+T$ capturing the fourth constraint in \eqref{eq:optimProb}. Second, the time instants at which robots $r_{ij}$ arrive at the projection of $\bbv\in \Omega_{\text{free}}^{|\ccalT_i|} $ to their own workspace, denoted by $t_{ij}(\bbv)$, are the same for all robots $r_{ij}\in\ccalT_i$, i.e., $\min_{r_{ij}\in\ccalT_i} t_{ij}(\bbv)=\max_{r_{ij}\in\ccalT_i} t_{ij}(\bbv)$ captures the sixth constraint in \eqref{eq:optimProb}. Third, the uncertainty metric evaluated at $\bbv$ should be less than a threshold $\delta$ capturing the last constraint in  \eqref{eq:optimProb}.

Algorithm \ref{alg:RRT} generates a tree denoted by $\ccalG_s=\{\ccalV_s,\ccalE_s\}$ that resides in $\Omega_{\text{free}}^{\abs{\ccalT_i}}$ where $\ccalV_s$ denotes its set of nodes and $\ccalE_s$ denotes its set of edges. This tree is initialized as $\ccalV_s = \set{\bbv_0}$, $\ccalE_s = \emptyset$ [line \ref{rrt:line1}, Alg. \ref{alg:RRT}] where the root $\bbv_0\in\Omega_{\text{free}}^{\abs{\ccalT_i}}$ of the tree is selected so that it matches the positions $\bbp_{ij}(t_{0,ij}^{k+T})$ of the robots $r_{ij}\in\ccalT_i$ in the joint space $\Omega_{\text{free}}^{\abs{\ccalT_i}}$. This way we ensure that the first constraint in \eqref{eq:optimProb} is satisfied. In other words, the tree generated by Algorithm \ref{alg:RRT} is rooted at the end point of the previous paths $\bbp_{ij}^{\bbark}$, where $k < \bar{k} \leq k+T-1$. The tree $\ccalG_s=\{\ccalV_s,\ccalE_s\}$ is built incrementally by adding new samples $\bbv\in\Omega_{\text{free}}^{\abs{\ccalT_i}}$ to $\ccalV_s$ and corresponding edges to $\ccalE_s$, based on three operations: `\texttt{Sample}' [line \ref{rrt:line3}, Alg. \ref{alg:RRT}], `\texttt{Extend}' [line \ref{rrt:extend}, Alg. \ref{alg:RRT}], and `\texttt{Rewire}' [line \ref{rrt:rewire}, Alg. \ref{alg:RRT}]. 
%
After taking $n_{\text{sample}}$ samples, where $n_{\text{sample}}>0$ is user-specified, Algorithm \ref{alg:RRT} terminates and returns the node $\bbv_{\text{end}} \in \ccalX_g^i$ with the smallest cost. Then, the path $\bbP_{\ccalT_i}^{k+T} = (\bbv_0, \dots, \bbv_{\text{end}})$ that connects $\bbv_{\text{end}}$ to the root $\bbv_0$ of the tree can be obtained [lines \ref{rrt:node}-\ref{rrt:jointP}, Alg. \ref{alg:RRT}]. The individual path segments $\bbp_{ij}^{k+T}$ of the robots are obtained by projecting the joint path $\bbP_{\ccalT_i}^{k+T}$ onto the workspace of each robot. Note that the paths start at different initial times $t^{k+T}_{0,ij}$ but end at the same final time $t^{k+T}_{f,\ccalT_i}$ since $\bbv_{\text{end}}$ belongs to the goal set. Figure \ref{fig:rrtEx} shows a schematic of the algorithm.

\begin{figure}[t]
  \centering
     \includegraphics[width=0.8\linewidth]{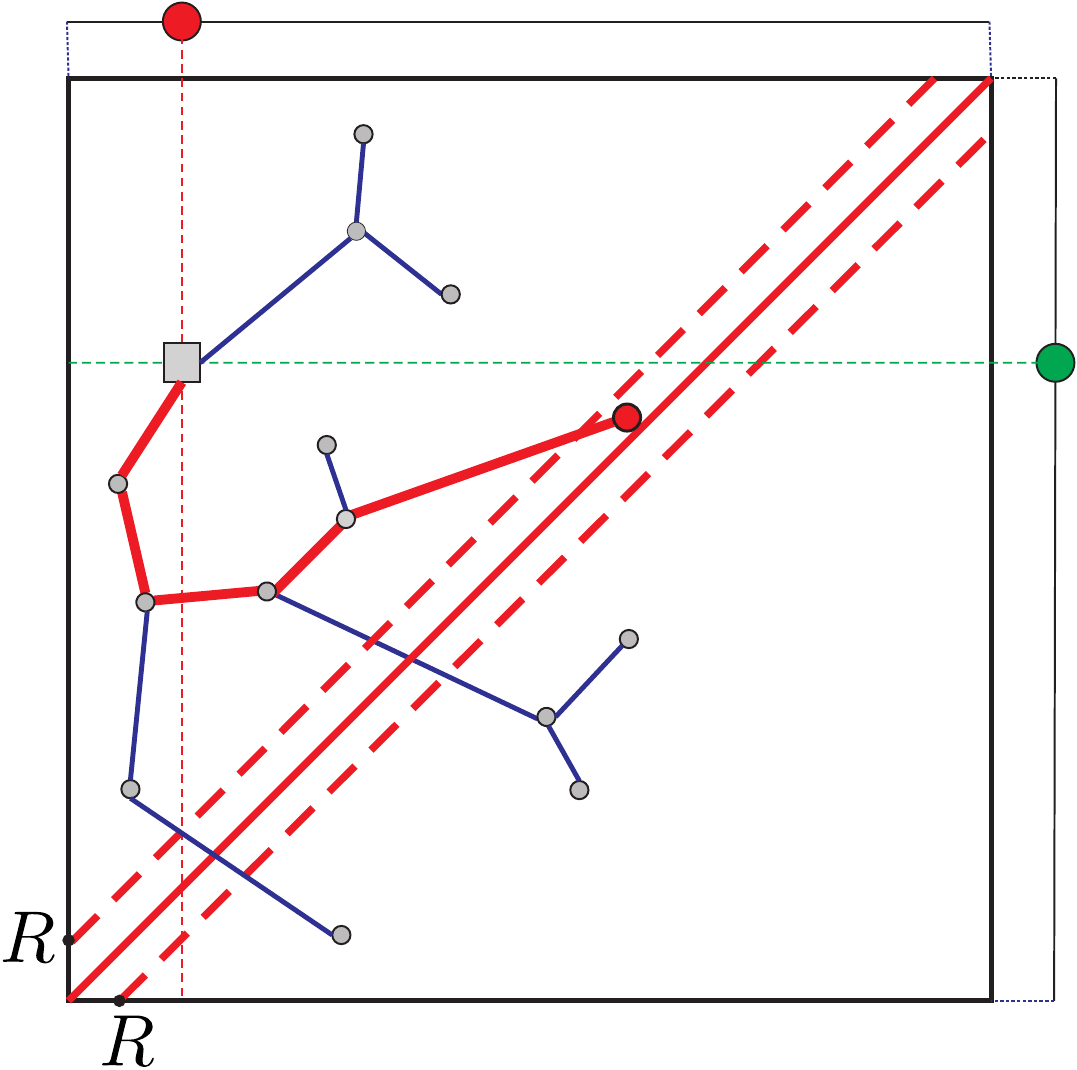}
  \caption{The schematic representation of a sampling-based solution for the planning problem \eqref{eq:optimProb}. Robots are illustrated by red and green disks residing in $1$-D workspace. This results in a $2$-D joint space $\Omega^{|\ccalT_i|}$ shown above. The gray square and circles in this space denote the root of the tree and the samples taken by algorithm. The goal set $\ccalX_g^i\subset\Omega^{|\ccalT_i|}$ is a subset of the area between the red dashed lines where the two robots are within communication range $R$. The time constraint (ii) in \eqref{goalset} depends on the initial times of the robots and determines the size of this subset. Red circle stands for a sample that lies within the goal set and is connected to the root through the red solid lines.}
  \label{fig:rrtEx}
\end{figure}

\subsubsection{Sampling function} \label{sec:sampFun}
 
At every iteration of Algorithm \ref{alg:RRT}, a new sample $\bbv_{\text{rand}}$ from the joint space $\Omega_{\text{free}}^{|\ccalT_i|}$ is taken [line \ref{rrt:line3}, Alg. \ref{alg:RRT}]. We assume that the first $B$ samples are drawn from a distribution $f_1:\Omega_{\text{free}}^{|\ccalT_i|}\rightarrow[0,1]$, e.g., a uniform distribution that forces the robots to explore $\Omega_{\text{free}}^{|\ccalT_i|}$, and then samples are drawn from a second distribution  $f_2:\Omega_{\text{free}}^{|\ccalT_i|}\rightarrow[0,1]$ that promotes the construction of a connected communication graph $\ccalG_{\ccalT_i}$. Specific choices for these sampling functions are discussed in Section \ref{sec:Sim}.

\subsubsection{Steer}\label{sec:steer1}
%

The next step in Algorithm \ref{alg:RRT} is to examine if the tree can be extended towards the new sample $\bbv_{\text{rand}}$. First, among all the nodes in the set $\ccalV_s$ we find the nearest node to $\bbv_{\text{rand}}$, which is denoted by $\bbv_{\text{nearest}}$ [line \ref{rrt:line4}, Alg. \ref{alg:RRT}]. Then, we define a steering function that given the robot dynamics \eqref{eq:rdynamics} returns a point $\bbv_{\text{new}}$ [line \ref{rrt:line5}, Alg. \ref{alg:RRT}], that (i) minimizes the distance $\left\|\bbv_{\text{new}}-\bbv_{\text{rand}}\right\|$, (ii) satisfies $\left\|\bbv_{\text{nearest}}-\bbv_{\text{new}}\right\|\leq\epsilon$, for some $\epsilon>0$, and (iii) minimizes the time difference $\max_{r_{ij}\in\ccalT_i}\bbv_{\text{new}}-\min_{r_{ij}\in\ccalT_i}\bbv_{\text{new}}$.\footnote{Depending on the robot dynamics \eqref{eq:rdynamics} and the arrival times of robots $r_{ij}\in\ccalT_i$ at their respective positions in the configuration $\bbv_{\text{nearest}}\in \Omega_{\text{free}}^{|\ccalT_i|}$, it may be impossible for the robots to arrive at the locations determined by $\bbv_{\text{new}}$ at the same time. Then, `\texttt{Steer}' minimizes the difference in arrival time.} Note that the steering function captures the dynamics of the robots \cite{karaman2010optimal}, i.e., the third constraint in \eqref{eq:optimProb}. The specific steering function used in this paper for robots with single integrator dynamics and bounded maximum velocities is discussed in Appendix \ref{sec:steer}.

\subsubsection{Extend}\label{sec:ext}

\begin{algorithm}[t]
\caption{\texttt{Extend}}
\label{alg:extend}
\begin{algorithmic}[1]
\STATE Set $\bbv_{\min} = \bbv_{\text{nearest}}$ and $\texttt{unc}_{\min}=\texttt{Cost}(\bbv_{\text{new}})$;\label{ext:line1}
\FOR{ $\bbv_{\text{near}} \in \ccalV_{\text{near}}$}\label{ext:line2}
	\STATE Compute $\texttt{Cost}(\bbv_{\text{new}})$ with $\bbv_{\text{near}}$ as the parent (Algorithm \ref{alg:kalman});
	 \IF { $\texttt{CollisionFree} (\bbv_{\text{near}}, \bbv_{\text{new}} ) ~\wedge~ \texttt{Cost}(\bbv_{\text{new}})< \texttt{unc}_{\min}$}\label{ext:line3}
		\STATE Set $\bbv_{\min} = \bbv_{\text{near}}$, $\texttt{unc}_{\min} = \texttt{Cost}(\bbv_{\text{new}})$;\label{ext:line4}
	\ENDIF 
\ENDFOR
\STATE Update the set of edges $\ccalE_s = \ccalE_s \cup \set{ (\bbv_{\min}, \bbv_{\text{new}}) } $;\label{ext:line6}
\STATE  $\texttt{Cost}(\bbv_{\text{new}})=\texttt{unc}_{\text{min}}$;\label{ext:cost}
\IF {$\bbv_{\text{new}} \in \ccalX_g^i$}\label{ext:line7}
	 \STATE Update $\ccalX_g^i = \ccalX_g^i \cup \set{  \bbv_{\text{new}} }$; \label{ext:line8}
\ENDIF
\end{algorithmic}
\end{algorithm}

Given the new point $\bbv_{\text{new}}$, Algorithm \ref{alg:RRT} next examines if the tree $\ccalG_s$ can be extended towards $\bbv_{\text{new}}$ [lines \ref{line:collisionFree}-\ref{rrt:extend}, Alg. \ref{alg:RRT}]. Specifically, first it examines if $\bbv_{\text{new}}$ can be reached from $\bbv_{\text{nearest}}$ through a path that satisfies the dynamics \eqref{eq:rdynamics} and lies in $\Omega_{\text{free}}$. This is checked using the function $\texttt{CollisionFree}(\bbv_{\text{nearest}}, \bbv_{\text{new}})$ [line \ref{line:collisionFree}, Alg. \ref{alg:RRT}].
If this is the case, then $\bbv_{\text{new}}$ is added to the set $\ccalV_s$  [line \ref{rrt:line6}, Alg. \ref{alg:RRT}] and we construct the set $\ccalV_{\text{near}} = \set{ \bbv \in \ccalV_s \, | \, \norm{\bbv - \bbv_{\text{new}}} < r }$, for some $r>0$ selected as in \cite{karaman2011sampling}, that collects all nodes in $\ccalV_s$ that are within distance $r$ from the point $\bbv_{\text{new}}$ [line \ref{rrt:line7}, Alg. \ref{alg:RRT}].
Then, among all the nodes in $\ccalV_{\text{near}}$, we pick the parent node of $\bbv_{\text{new}}$, denoted by $\bbv_{\text{min}}$, that incurs the minimum possible cost for the node $\bbv_{\text{new}}$ [line \ref{rrt:extend}, Alg. \ref{alg:RRT}].
This process is described in Algorithm \ref{alg:extend}.

Algorithm \ref{alg:extend} is initialized with $\bbv_{\text{min}}=\bbv_{\text{nearest}}$ [line \ref{ext:line1}, Alg. \ref{alg:extend}]. The cost of $\bbv_{\text{new}}$ given the parent node $\bbv_{\text{min}}$ is denoted by $\texttt{unc}_{\text{min}}=\texttt{Cost}(\bbv_{\text{new}})$ and its calculation is discussed in Section \ref{sec:cost}. 
Then, for every candidate parent node $\bbv_{\text{near}}\in\ccalV_{\text{near}}$ of $\bbv_{\text{new}}$, Algorithm \ref{alg:extend} checks if there is a collision-free path that connects $\bbv_{\text{new}}$ to $\bbv_{\text{near}}$ and calculates the cost  $\texttt{Cost}(\bbv_{\text{new}})$. If $\texttt{Cost}(\bbv_{\text{new}})\leq \texttt{unc}_{\text{min}}$, then the parent node of  $\bbv_{\text{new}}$ is updated as  $\bbv_{\min} = \bbv_{\text{near}}$ and the corresponding cost is updated as $\texttt{unc}_{\min} = \texttt{Cost}(\bbv_{\text{new}})$. This process is repeated for all nodes $\bbv_{\text{near}}\in\ccalV_{\text{near}}$.
%
Once all nodes $\bbv_{\text{near}}\in\ccalV_{\text{near}}$ are examined, the parent $\bbv_{\text{min}}$ of $\bbv_{\text{new}}$ is selected and we update the set of edges as $\ccalE_s = \ccalE_s \cup \set{ (\bbv_{\min}, \bbv_{\text{new}}) }$ [line \ref{ext:line6}, Alg. \ref{alg:extend}]. If $\bbv_{\text{new}}$ satisfies the three conditions of the goal set, specified in \eqref{goalset}, then we update the set as $\ccalX_g^i = \ccalX_g^i \cup \set{\bbv_{\text{new}}}$ [line \ref{ext:line8}, Alg. \ref{alg:extend}]. 

\subsubsection{Rewire}\label{sec:rew}

\begin{algorithm}[t]
\caption{\texttt{Rewire}}
\label{alg:rewire}
\begin{algorithmic}[1]
\FOR{ $\bbv_{\text{near}} \in \ccalV_{\text{near}}$ }\label{rew:line1}
				\IF { $\texttt{CollisionFree} (\bbv_{\text{new}},\bbv_{\text{near}})$ $~\wedge~$ $\min_{r_{ij}\in\ccalT_i}t_{ij}(\bbv_{\text{near}}) = \max_{r_{ij}\in\ccalT_i}t_{ij}(\bbv_{\text{near}})$}\label{rew:line2}
				    \STATE Compute cost $\overline{\texttt{Cost}}(\bbv_{\text{near}})$ with $\bbv_{\text{new}}$ as the parent (Algorithm \ref{alg:kalman});\label{rew:line3}
				    \IF {$\overline{\texttt{Cost}}(\bbv_{\text{near}}) \leq \texttt{Cost}(\bbv_{\text{near}})$ $~\wedge~$ $\min_{r_{ij}\in\ccalT_i}t_{ij}(\bbv_{\text{near}})=\max_{r_{ij}\in\ccalT_i}t_{ij}(\bbv_{\text{near}})$}		\label{rew:line4}
				    \STATE Update the set of edges $\ccalE_s = \ccalE_s \setminus \set{ (\bbv_{\text{parent}}, \bbv_{\text{near}} ) } \cup \set{ (\bbv_{\text{new}}, \bbv_{\text{near}}) } $;\label{rew:line5}
				    \ENDIF
				\ENDIF 
				
			\STATE Update the goal set $\ccalX_g^i$;\label{rew:updGoal} 
\ENDFOR		    
\end{algorithmic}
\end{algorithm} 

After extending the tree $\ccalG_s$ towards $\bbv_{\text{new}}$ the rewiring process follows that checks if it is possible to further reduce the cost of the nodes of the tree by rewiring them through $\bbv_{\text{new}}$ [line \ref{rrt:rewire}, Alg. \ref{alg:RRT}]. The rewiring operation is described in Algorithm \ref{alg:rewire}. In particular, for every node $\bbv_{\text{near}}\in\ccalV_{\text{near}}$ that (i) satisfies $\min_{r_{ij}\in\ccalT_i} t_{ij}(\bbv_{\text{near}})=\max_{r_{ij}\in\ccalT_i} t_{ij}(\bbv_{\text{near}})$, and (ii) can be connected through an obstacle-free path to $\bbv_{\text{new}}$, we compute $\overline{\texttt{Cost}}(\bbv_{\text{near}})$ assuming that $\bbv_{\text{near}}$ was connected to $\bbv_{\text{new}}$ [lines \ref{rew:line1}-\ref{rew:line3}, Alg. \ref{alg:rewire}]. Then we rewire $\bbv_{\text{near}}$ if the cost $\overline{\texttt{Cost}}(\bbv_{\text{near}})$ using $\bbv_{\text{new}}$ as its parent is less than the current cost $\texttt{Cost}(\bbv_{\text{near}})$ and if there are control inputs $\bbu_{ij}$ that can still drive all robots to
$\bbv_{\text{near}}$ at the same time [line \ref{rew:line4}, Alg. \ref{alg:rewire}]. 
If both conditions are satisfied, then we update the set of edges $\ccalE_s$ by deleting the previous edge $(\bbv_{\text{parent}},\bbv_{\text{near}})$, where  $\bbv_{\text{parent}}$ stands for the previous parent of $\bbv_{\text{near}}$, and adding the new edge  $(\bbv_{\text{new}},\bbv_{\text{near}})$ [line \ref{rew:line5}, Alg. \ref{alg:rewire}]. Notice that the requirement that a node $\bbv_{\text{near}}\in\ccalV_{\text{near}}$ should satisfy $\min_{r_{ij}\in\ccalT_i} t_{ij}(\bbv_{\text{near}})=\max_{r_{ij}\in\ccalT_i} t_{ij}(\bbv_{\text{near}})$ to get rewired does not exist in the the $\text{RRT}^{*}$ algorithm \cite{karaman2011sampling}. 

If a node $\bbv_{\text{near}}$ is rewired then we update the cost of all successor nodes of $\bbv_{\text{near}}$ in $\ccalV_s$. Also, notice that after rewiring a node $\bbv_{\text{near}}$, the time instant $t_{ij}(\bbv_{\text{near}})$ associated with the arrival of robots $r_{ij}$ at $\bbv_{\text{near}}$ may change. As a result, $\texttt{Cost}(\bbv_{\text{near}})$ may change, since the cost $\texttt{Cost}(\bbv)$ depends on the time instants at which the robots $r_{ij}\in\ccalT_i$ are physically present at their respective positions in the configuration $\bbv\in \Omega_{\text{free}}^{|\ccalT_i|}$; see Section \ref{sec:cost}. Consequently the cost of all successor nodes of $\bbv_{\text{near}}$ may change as well. Thus, after rewiring, the goal set needs to be updated, by checking if $\bbv_{\text{near}}$ and its successor nodes satisfy the relevant conditions [line \ref{rew:updGoal}, Algorithm \ref{alg:rewire}].

\subsection{Computation of Cost Function}\label{sec:cost}
In this section, we discuss the computation of the cost $\texttt{Cost}(\bbv_{\text{child}})$ of a node $\bbv_{\text{child}}$ given a candidate parent $\bbv_{\text{parent}}$ required in Algorithms \ref{alg:extend} and \ref{alg:rewire}. Notice that $\texttt{Cost}(\bbv_{\text{child}})$ corresponds to the cost of the path that connects the node $\bbv_{\text{child}}$ to the root of the tree $\bbv_0$.  
As our uncertainty metric, we use the maximum eigenvalue of the posterior covariance $\bbC_{\ccalT_i}(t)$ computed using only local information available to team $\ccalT_i$, i.e., $\texttt{unc}(\bbP_{\ccalT_i}(t)) = \lambda_n(\bbC_{\ccalT_i}(t))$, where $\lambda_n(\cdot)$ denotes maximum eigenvalue. Keeping in mind that the posterior estimate $\hat{\bbx}_{\ccalT_i}(t)$ of $\bbx(t)$ and covariance $\bbC_{\ccalT_i}(t)$ are obtained using only the local measurements available to team $\ccalT_i$, we drop the subscript $\ccalT_i$ from these notations for simplicity. To obtain the posterior estimate, here we use the Extended Kalman Filter (EKF) to fuse the measurements that will be collected along the paths of the robots. The details are given in Algorithm \ref{alg:kalman}. Specifically, to compute $\texttt{Cost}(\bbv_{\text{child}})$, we first construct a finite sequence $\ccalY=\left\{(t_1, \bbq_1),\dots,(t_2, \bbq_2),\dots,(t_{\Xi}, \bbq_{\Xi}) \right\}$, where 
\begin{equation} \label{eq:intervalKF}
t_1 = \min_{r_{ij}\in\ccalT_i} t_{ij}(\bbv_{\text{parent}}) \ \ \text{and} \ \	t_{\Xi} = \max_{r_{ij}\in\ccalT_i} t_{ij}(\bbv_{\text{child}}),
\end{equation}
that collects the time instants $t_{\xi}$ and corresponding locations $\bbq_{\xi}\in\Omega_{\text{free}}$ where the robots in team $\ccalT_i$ will take measurements, following their projected paths from $\bbv_{\text{parent}}$ to $\bbv_{\text{child}}$. We assume that the robots take measurements every $\Delta t \in \mbN$. 
%
%

%

The cost of node $\bbv_{\text{child}}$ is initialized as $\texttt{Cost}(\bbv_{\text{child}})=\texttt{Cost}(\bbv_{\text{parent}})$ [line \ref{kalman:line1}, Alg. \ref{alg:kalman}]. Then, using the EKF from $t_{1}$ until $t_{\Xi}$ we update the cost of $\bbv_{\text{child}}$ by fusing sequentially all predicted measurements that will be taken at locations and time instants collected in $\ccalY$.
Specifically, during the time intervals $(t_{\xi-1},t_{\xi})$, we execute the prediction stage \eqref{eq:predKF} of the EKF to compute the predicted covariance matrix $\bbC(t)$ of the estimate $\hat{\bbx}(t)$. Given the predicted covariance matrix $\bbC(t)$, we compute $\texttt{unc}(\bbP_{\ccalT_i}(t))$, where $\bbP_{\ccalT_i}(t)\in\Omega_{\text{free}}^{|\ccalT_i|}$ is the joint position of robots in team $\ccalT_i$ at time $t$. Then, we update the cost of $\bbv_{\text{child}}$ as $\texttt{Cost}(\bbv_{\text{child}})=\texttt{Cost}(\bbv_{\text{child}})+\texttt{unc}(\bbP_{\ccalT_i}(t))$, for all discrete $t\in(t_{\xi-1},t_{\xi})$ [lines \ref{kalman:line6}-\ref{kalman:line8}, Alg. \ref{alg:kalman}]. At time step $t_{\xi}$, using the predicted measurement that would be taken at this time instant, we compute the estimated covariance matrix $\bbC(t_{\xi})$ [lines \ref{kalman:line10}-\ref{kalman:line12}, Alg. \ref{alg:kalman}]. Given $\bbC(t_{\xi})$, we compute $\texttt{unc}(\bbP_{\ccalT_i}(t_{\xi}))$ and update the cost of node $\bbv_{\text{child}}$ as
 $\texttt{Cost}(\bbv_{\text{child}})=\texttt{Cost}(\bbv_{\text{child}})+\texttt{unc}(\bbP_{\ccalT_i}(t_{\xi}))$ [line \ref{kalman:line15}, Alg. \ref{alg:kalman}]. This procedure is repeated for all intervals $(t_{\xi},t_{\xi+1})$.

It is possible that multiple concurrent measurements are taken by robots. In this case, these measurements must be fused together in Algorithm \ref{alg:kalman}. Moreover, during planning, the actual measurements $\bby(t_{\xi})$ in \eqref{eq:inovMeas} are unavailable and the predicted state will directly be used in \eqref{eq:fuseCov}.

 
\begin{algorithm}[t]
\caption{Extended Kalman Filter}
\label{alg:kalman}
\begin{algorithmic}[1]
\REQUIRE Dynamics \eqref{eq:dynamics} and observation \eqref{eq:measModel} models;
\REQUIRE Finite sequence of measurements $\ccalY$;
\REQUIRE Process noise covariance $\bbQ(t)$ and measurement noise covariance $\bbR(t)$;
\REQUIRE Initial estimates $\hat{\bbx}(t_{\text{near}})$ and $\bbC(t_{\text{near}})$;	
%
\STATE Initialize $\texttt{Cost}(\bbv_{\text{child}})=\texttt{Cost}(\bbv_{\text{parent}})$; \label{kalman:line1}
\STATE Set $\xi=2$;\label{kalman:line3} 
\WHILE{ $\xi\leq \Xi$}\label{kalman:line5}
	\FOR{ $t = t_{\xi-1}:t_{\xi} $ }\label{kalman:line6}
		\STATE Compute the Jacobians $\bbF(t-1) = \nabla_{\bbx} \bbf|_{\hat{\bbx}(t-1), \bbu(t-1)}$ and $\bbL(t-1) = \nabla_{\bbw} \bbf|_{\hat{\bbx}(t-1), \bbu(t-1)}$;
		\STATE Compute the EKF prediction:\label{kalman:line7}
		\begin{subequations} \label{eq:predKF}
		\begin{align}
			&\hat{\bbx}(t) = \bbf(\hat{\bbx}(t-1),\bbu(t-1) ) \\
			&\bbC(t) = \bbF(t-1) \, \bbC(t-1) \, \bbF(t-1)^T + \\ \nonumber
			& \ \  \ \ \ \ \ \ \ \, \bbL(t-1) \bbQ(t-1) \bbL(t-1)^T \\
			&\texttt{Cost}(\bbv_{\text{child}})=\texttt{Cost}(\bbv_{\text{child}})+\texttt{unc}(\bbP_{\ccalT_i}(t))	\label{eq:predCov}
		\end{align}
		\end{subequations}
	\ENDFOR\label{kalman:line8}
\STATE Compute the Jacobians $\bbH(t_{\xi}) = \nabla_{\bbx} \bbh|_{\hat{\bbx}(t_{\xi}), \bbq_{\xi}}$ and $\bbM(t_{\xi}) = \nabla_{\bbv} \bbh|_{\hat{\bbx}(t_{\xi}), \bbq_{\xi}}$;
\STATE Compute Innovations: \label{kalman:line10}
	\begin{subequations}
	\begin{align}
	\bar{\bby}(t_{\xi}) &= \bby(t_{\xi}) - \bbh(\hat{\bbx}(t_{\xi}), \bbq_{\xi}) \label{eq:inovMeas} \\
	\bbS(t_{\xi}) &= \bbH(t_{\xi}) \, \bbC(t_{\xi}) \, \bbH(t_{\xi})^T +  \bbM(t_{\xi}) \bbR(t_{\xi}) \bbM(t_{\xi})^T
	\end{align}
	\end{subequations}
	\STATE Kalman gain: $\bbK(t_{\xi}) = \bbC(t_{\xi}) \bbH(t_{\xi})^T \, \bbS(t_{\xi})^{-1}$;	\label{kalman:line11}
	\STATE Compute the EKF estimates:	\label{kalman:line12}
	\begin{subequations}
	\begin{align}
	\hat{\bbx}(t_{\xi}) &= \hat{\bbx}(t_{\xi}) + \bbK(t_{\xi}) \barby(t_{\xi})\\
	\bbC(t_{\xi})& = \left[ \bbI_n - \bbK(t_{\xi}) \, \bbH(t_{\xi}) \right] \, \bbC(t_{\xi})	\label{eq:fuseCov}
	\end{align}
	\end{subequations}
	\STATE Update $\texttt{Cost}(\bbv_{\text{child}})=\texttt{Cost}(\bbv_{\text{child}})+\texttt{unc}(\bbP_{\ccalT_i}(t_{\xi}))$;	\label{kalman:line15}
\ENDWHILE
\end{algorithmic}
\end{algorithm}   
%
 

\begin{rem} [Alternative Estimation Filters]
In Algorithm \ref{alg:kalman} we use the EKF to fuse the measurements collected by the robots but any other estimation filter could be used. Particularly, the Information form of the Kalman filter is attractive for distributed fusion since the information from different measurements is additive, cf. \cite{DFDSN2001DS}. For the case of out-of-order data, the Information form requires storage of an information vector, with length $n$, corresponding to each measurement \cite{DADDSN2001ND} and the Kalman form requires storage of the time, location, and value of each measurement, i.e., a vector of length $(1+d+m)$. Therefore, the selection of the appropriate filter is problem-dependent.
For networks with limited resources, there exist a family of DSE algorithms that reduce the communication and computational overhead by performing inexact robust fusion of the out-of-order data; see \cite{DDF2016CA} and the references therein. Other approximate methods utilize low rank representations of the covariance matrix for efficient data fusion \cite{CSKFNSE2015LKGD,RTDALSS2015GKLD}. Application of such methods can particularly improve the efficiency of path planning without sacrificing accuracy. Furthermore, the present value of delayed information diminishes as the delay increases \cite{DADDSN2001ND}; thus, the robots can stop communicating old information to further conserve bandwidth and memory if necessary. Other heuristics of this kind may be adapted depending on application; see, e.g., \cite{DDFMN2005MC}.
\end{rem}

\subsection{Completeness and Optimality}
 
Our proposed sampling-based algorithm is probabilistically complete, i.e., if there exist paths $\bbp_{ij}^{k+T}$ that terminate at the goal set \eqref{goalset}, then Algorithm \ref{alg:RRT} will find them with probability 1, as $n_{\text{sample}}\to\infty$. To show this result, recall that $\text{RRT}^{*}$ is probabilistically complete given the functions `\texttt{Steer}' and `\texttt{Extend}' \cite{karaman2011sampling}. The only requirement in the `\texttt{Steer}' function is that the node $\bbv_{\text{new}}$ is closer to $\bbv_{\text{rand}}$ than $\bbv_{\text{nearest}}$ is, which is trivially satisfied by Algorithm \ref{alg:RRT}. Finally since the `\texttt{Extend}' function, described in Algorithm \ref{alg:extend}, is the as same as the extend operation of $\text{RRT}^{*}$ \cite{karaman2011sampling}, we conclude that Algorithm \ref{alg:RRT} is probabilistically complete.
 
Nevertheless, Algorithm \ref{alg:RRT} is not asymptotically optimal, since rewiring a node that belongs to $\ccalV_{\text{near}}$ takes place only if its cost after rewiring decreases and all robots can still arrive at this node simultaneously. On the other hand, in the rewiring step of $\text{RRT}^*$ the time constraint considered here is not present. To recover the asymptotic optimality of the $\text{RRT}^{*}$ we can relax the time constraint (ii) in the goal set \eqref{goalset} and perform the rewiring step as in the $\text{RRT}^*$ algorithm. As a result, the robots will not form a connected graph simultaneously. 
Requiring all robots to form a connected graph at the same time has two benefits. First, it prevents the robots from waiting at a fixed location which is suboptimal for the estimation of a dynamic process. Second aligning the time stamp of measurements of the agents during sampling, speeds up the fusion process and alleviates the need for storing measurement information every time a new node is added to the tree. To see this, note that in order to compute the cost of a new sample $\bbv_{\text{child}}$, we fuse the measurements in the interval $(t_1,t_{\Xi})$ defined by \eqref{eq:intervalKF}. The length of this interval increases when the time constraint (ii) is not enforced. Moreover, the measurements associated with time instants that belong to the interval $(\min_{r_{ij}\in\ccalT_i}t_{ij}(\bbv_{\text{child}}), \, t_{\Xi})$ need to be retained for the computation of the cost function for the future samples whose parent is $\bbv_{\text{child}}$.

\subsection{Integrated Path Planning and Intermittent Communication Control} \label{sec:}
The integrated algorithm is described in Algorithm \ref{alg:dhc}. Every robot $r_{ij}$ follows the path $\bbp_{ij}^k$ while collecting measurements [line \ref{dhc:line3}, Alg. \ref{alg:dhc}]. \footnote{Before the first communication event, robots in team $\ccalT_i$ have not designed yet any paths $\bbp_{ij}^k$. The paths that the robots of team $\ccalT_i$ should follow to participate in their first communication event are \textit{a priori} planned and are denoted by $\bbp_{ij}^{\text{init}}$.}
When robot $r_{ij}$ reaches the final waypoint of that path segment, all robots in team $\ccalT_i$ form a connected communication network, simultaneously. When this happens, the robots exchange the information they have collected since their last communication event. This information can be due to new measurements or communication with other teams $\ccalT_j$ [line \ref{dhc:line5}, Alg. \ref{alg:dhc}].
Particularly, in order to identify duplicate information, each robot constructs and maintains a local copy of the information that it owns. This process is very similar to the concept of information trees in \cite{DDFMN2005MC}.
Given the new set of measurements, the state estimate and the respective covariance are updated using the Extended Kalman Filter [line \ref{dhc:line6}, Alg. \ref{alg:dhc}] given in Algorithm \ref{alg:kalman}. 
Finally, the robots in team $\ccalT_i$ compute the paths $\bbp_{ij}^{k+T}$ that will allow them to reconnect at epoch $k+T$ [line \ref{dhc:line7}, Alg. \ref{alg:dhc}]. 

\begin{rem} [Memory Requirement]
From Proposition \ref{prop:kstar}, it follows that information collected by any robot $r_{ij}$ will be propagated to all teams within at most $D_{\ccalG_{\ccalT}}$ epochs. Thus, information collected in earlier epochs can be discarded.
\end{rem}


\begin{rem} [Exogenous Disturbances]
Note that exogenous disturbances may affect the travel times of the robots. To account for such delays, we can require each robot $r_{ij}\in\ccalT_i$ that finishes the execution of its path to wait until all other robots in team $\ccalT_i$ also complete their paths. We can show that under this control policy, the network is \textit{deadlock-free} meaning that there are no robots in any team $\ccalT_i$ that are waiting forever to communicate with other robots in team $\ccalT_i$; see Proposition 6.2 in \cite{kantaros2017temporal} for more details.
\end{rem}

\begin{algorithm}[t]
\caption{Integrated Control Framework for Robot $r_{ij}$}
\label{alg:dhc}
\begin{algorithmic}[1]
\REQUIRE Initial paths $\bbp_{ij}^{\text{init}}$ and the plan $\texttt{sched}_{ij}$;
\FOR {$k=1:\infty$}\label{dhc:line1}
				\STATE Move along the path $\bbp_{ij}^k$ and collect measurements every $\Delta t$;\label{dhc:line3}
				\STATE When the final waypoint of $\bbp_{ij}^k$ is reached, exchange information with all other robots $r_{ij}\in\ccalT_i$;\label{dhc:line5}
				\STATE Update the state estimate and the covariance matrix up to the current time $t$ using Algorithm \ref{alg:kalman};\label{dhc:line6}
				\STATE Compute the paths $\bbp_{ij}^{k+T}$ to be followed to reach the next meeting event;\label{dhc:line7}
\ENDFOR
\end{algorithmic}
\end{algorithm} 

\section{Numerical Experiments} \label{sec:Sim}
The proposed state estimation framework, described in Algorithm \ref{alg:dhc}, can be applied to any state estimation problem,  such as target tracking \cite{hollinger2009efficient}, estimation of spatio-temporal fields \cite{lan2016rapidly}, or state estimation in distributed parameter systems \cite{meC1,meC4,muraca2009state}, as long as the cost function $\texttt{Cost}(\bbP_{\ccalT_i})$ is additive and monotone.
In this section, we demonstrate the performance of the proposed algorithm for a target tracking problem in a non-convex environment. Specifically, we consider $N=8$ robots and $8$ targets that reside in a $10 \times 10 \times 5$m$^3$ indoor environment, shown from top in Figure \ref{fig:paths}. The targets move in this $3$D environment while avoiding the obstacles and are modeled by a linear time-invariant dynamics
$$ \bbx_a(t+1) = \bbA_a \bbx_a(t) + \bbB_a \bbu_a (t) + \bbw_a(t), $$
where $a$ denotes the target index, $\bbx_a (t) \in \mathbb{R}^3$, and $\bbw_a (t) \sim (\bb0, \bbQ_a(t) )$, where $\bbQ_a(t)\in\mathbb{S}^3_{++}$ denotes the covariance matrix of the process noise associated with $a$-th target. For simplicity, we assume that the targets follow linear or circular paths corrupted by process noise.
We also consider ground robots that live in $\Omega$, which is the $2$D projection of the space of targets. The robots are governed with the following first-order dynamics
$$ \bbp_{ij}(t+1)=\bbp_{ij}(t)+\bbu_{ij}(t), $$
where $\bbp_{ij}(t)\in\mathbb{R}^2$, $\bbu_{ij}\in\mathbb{R}^2$, and $\lVert \bbu_{ij}\rVert \leq 0.1$m. The communication range of the robots is $R=0.2$m. Observe that the communication radius is small compared to the size of the workspace. The robots need to move in $\Omega_{\text{free}}$ to estimate the $3$D positions  $\bbx_a(t)$ of all targets.

We equip the robots with omnidirectional, \range-only, line-of-sight sensors with limited range of $5$m.
Every robot can take noisy measurements of its distance from all targets that lie within its sight and range. Specifically, the measurement associated with target $a$ is given by
\begin{equation}\label{eq:meassim}
y_{ij,a} = \ell_{ij,a}(t) + v(t) \ ~\mbox{if}~ \ \ell_{ij,a}(t) \leq 5 ,
\end{equation}
where $\ell_{ij,a}(t) = ||\bbp_{ij}(t) - \bbx_a(t)||$, $v(t) \sim \ccalN(0,\sigma^2(t))$, and 
\begin{align}
\label{eq:sigma}
\sigma(t)=\left\{
                \begin{array}{ll}
                   0.01, 					&\mbox{if}~ \ell_{ij,a}(t) \leq 1,\\
                   0.045 \, \ell_{ij,a}(t) - 0.035, 	&\mbox{if}~1< \ell_{ij,a}(t) \leq 3, \\
                   0.1, 					&\mbox{if}~3 < \ell_{ij,a}(t) \leq 5 .
                \end{array}
              \right.
\end{align}
The standard deviation $\sigma(t)$ is very small, i.e., $0.01$m, when the robot $r_{ij}$ and target $a$ are less than $1$m apart and then linearly increases until it becomes flat at $0.1$m between $3$m to $5$m; see also Figure \ref{fig:sigma}. This model captures the fact that the range readings become less accurate as the distance increases and is designed to motivate the robots to approach the targets.

\begin{figure}[t]
\centering
\includegraphics[width=0.75\linewidth]{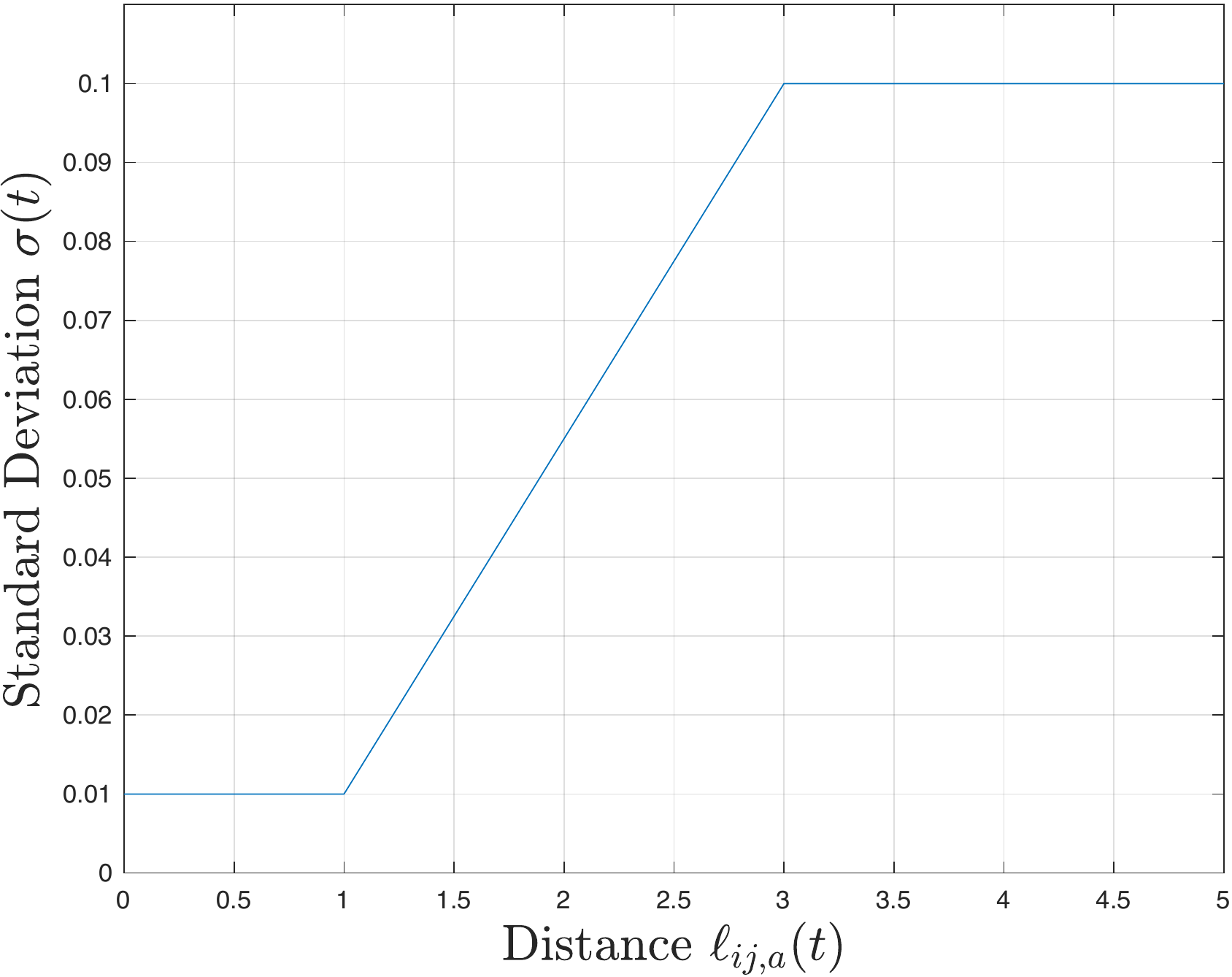}
\caption{Graphical depiction of the standard deviation $\sigma(t)$ defined in \eqref{eq:sigma}.}
\label{fig:sigma}
\end{figure}

%




The initial positions of the targets are stacked in the vector
\begin{align*}
\bbx_0 = &[(1,    2,  1.5),  (1 , 1,  1.3),  (8.5,   1.5,    0.8),  (9.2,  9,    1), \\
& (7,  8,  1.1), (5.9,   9,   1.2), (3,  9, 1.3), (1.8,    8.2,    0.9)]^T .
\end{align*}
The Extended Kalman Filter (EKF) is initialized as\footnote{Note that the EKF can be initialized arbitrarily. However, the more inaccurate the initialization is, the worse the localization performance is, as the EKF relies on the linearization of the nonlinear measurement model \eqref{eq:meassim} around the estimate $\hat{\bbx}$.}
\begin{align*}
\hat{\bbx}_0 = &[(1 ,  1,   1),  (1,   1,    1),  (9,   1,   1),  (9,  9.5,   1), \\
& (7,    9.5,    1), (4,    9,   1), (2.5,    7,    1), (1,  9,   1)]^T ,
\end{align*}
and the respective covariance matrix $\bbC(0)\in\mathbb{S}_{++}^{24}$ is initialized to be a diagonal matrix with diagonal entries equal to $0.25$m$^2$.

As mentioned in Section \ref{sec:cost}, we select the cost function as $\texttt{unc}(\bbP_{\ccalT_i}(t)) = \lambda_n(\bbC_{\ccalT_i}(t))$, where $\lambda_n(\cdot)$ denotes maximum eigenvalue and $\bbC_{\ccalT_i}\in\mathbb{S}_{++}^{24}$ is the covariance matrix computed using only local information from team $\ccalT_i$.\footnote{The dimension of $\bbC_{\ccalT_i}$ is $24$ as it pertains to $8$ targets that reside in $\mathbb{R}^3$.} In this way, we minimize the worst-case uncertainty of localizing the targets.
The goal set is constructed as in \eqref{goalset} where for the problem at hand, the constraint (iii) in \eqref{goalset} is defined as $\lambda_n(\bbC_{\ccalT_i}^{a_i}(t_{f,\ccalT_i})) \leq \delta$, in which $a_i$ is the most uncertain target from the view point of team $\ccalT_i$ and $\bbC_{\ccalT_i}^{a_i}(t)$ is the $a_i$-th diagonal block in the covariance matrix $\bbC_{\ccalT_i}$. In other words, we require each team $\ccalT_i$ to design paths so that the uncertainty of the most uncertain target, determined according to the local information of team $\ccalT_i$, drops below a threshold $\delta$.  
In what follows, we select  $\delta=0.12^2$m$^2$, i.e., team $\ccalT_i$ must localize the most uncertain target with uncertainty not worse than $0.12$m. If the team does not find a feasible path that satisfies this constraint for the given number of samples $n_{\text{sample}}$, then the constraint (iii) in the goal set \eqref{goalset} is checked for the next most uncertain targets. 

Finally, the sampling function defined in Section \ref{sec:sampFun}, is constructed as follows. First, we draw a pre-defined number of samples from a uniform distribution $f_1$ defined over $\Omega_{\text{free}}^{|\ccalT_i|}$ that allows the team $\ccalT_i$ to explore the domain. Then, we switch to a distribution $f_2$ that forces the robots to get closer and form a connected configuration. This distribution is constructed in two steps. Specifically, we first draw a possible meeting location $\bbq \in \Omega_{\text{free}}$ from a uniform distribution. Then, we construct the joint vector $\boldsymbol{\mu} = [\bbq^T,\dots,\bbq^T]^T\in\Omega_{\text{free}}^{|\ccalT_i|}$ and draw the desired samples from the following normal distribution $f_2(\bbv)=\ccalN(\boldsymbol{\mu},(2R)^2 \bbI)$.

In what follows, in Section \ref{sim1}, we examine the performance of our algorithm for two different configurations of the graph $\ccalG_{\ccalT}$. Then, in Section \ref{sim2}, we compare our proposed algorithm to an algorithm that preserves end-to-end network connectivity for all time and a heuristic approach that allows disconnection of the network, similar to the approach proposed here, but selects randomly a meeting location for every team, where the robots in that team should meet to communicate. The first comparison shows the significant advantage of the framework presented here against the extensive literature on distributed estimation \cite{DADDSN2001ND,KFIO2004SSFPJS,DRSSIA2012JASR} that requires all-time connectivity. The second comparison is meant to justify the use of the sampling-based Algorithm \ref{alg:RRT} to solve the planning problem \eqref{eq:optimProb} instead of a simple heuristic.

\subsection{Comparative Results for Different Graphs $\ccalG_\ccalT$}\label{sim1}
First, we assume that the network of $N=8$ robots is divided into $M=8$ teams defined as $\ccalT_1=\{r_{12},r_{18}\}$, $\ccalT_2=\{r_{12},r_{23}\}$, $\ccalT_3=\{r_{23}, r_{34}\}$, $\ccalT_4=\{r_{14},r_{45}\}$, $\ccalT_5=\{r_{56},r_{45}\}$, $\ccalT_6=\{r_{56},r_{67}\}$, $\ccalT_7=\{r_{67},r_{78}\}$, $\ccalT_8=\{r_{18},r_{78}\}$. The resulting graph, denoted by $\ccalG_{\ccalT}^1$, is shown in Figure \ref{fig:GT1} and is connected. Given the decomposition of the network into $M=8$ teams, we construct the following communication schedules $\texttt{sched}_{ij}$ for all robots $r_{ij}$, as discussed in Section \ref{sec:DP}:

\begin{align*} 
\left[\begin{array}{c}
\texttt{sched}_{12}\\
\texttt{sched}_{23}\\
\texttt{sched}_{34}\\
\texttt{sched}_{45}\\
\texttt{sched}_{56}\\
\texttt{sched}_{67}\\
\texttt{sched}_{78}\\
\texttt{sched}_{81}
\end{array}\right]=&
\left[\begin{array}{cccccc}
     1     &  2    \\
     3   &  2    \\
     3    &  4    \\
     5    &  4    \\
     5    &  6    \\
     7    &  6    \\
     7    &  8   \\
     1     &  8   
\end{array}\right]^{\omega}.
\end{align*}

Notice that the period of these schedules is $T=2$. At epoch $k=1$ the robots in teams $\ccalT_1$, $\ccalT_3$, $\ccalT_5$, $\ccalT_5$, $\ccalT_7$ communicate and decide on their next paths so that they re-connect after $T=2$ epochs, i.e., at epoch $k=3$. 

For comparison, we also consider a second denser graph $\ccalG_{\ccalT}^2$ with $M=5$ teams which is shown in Figure \ref{fig:GT2}. In this case, the robots are organized into the following teams, $\ccalT_1=\{r_{12}, r_{14}, r_{15} \}$, $\ccalT_2=\{r_{12}, r_{23}, r_{25} \}$, $\ccalT_3=\{r_{23}, r_{34}, r_{35} \}$, $ \ccalT_4=\{r_{14}, r_{34}, r_{45}\}$, and $\ccalT_5=\{r_{15}, r_{25}, r_{35}, r_{45} \}$, and the communication schedules $\texttt{sched}_{ij}$ have the following form:

\begin{align*} 
\left[\begin{array}{c}
\texttt{sched}_{12}\\
\texttt{sched}_{23}\\
\texttt{sched}_{34}\\
\texttt{sched}_{14}\\
\texttt{sched}_{15}\\
\texttt{sched}_{25}\\
\texttt{sched}_{35}\\
\texttt{sched}_{45}
\end{array}\right]=
&\left[\begin{array}{cccccc}
    1  &  2& X   \\
    3  &  2  & X  \\
    3 &  4 & X  \\
    1  &  4  & X  \\
    1  &  X & 5    \\
    X & 2 &  5    \\
    3 & X &  5    \\
    X & 4 & 5    
\end{array}\right]^{\omega}.
\end{align*}



\begin{figure}[t]
\centering
\includegraphics[width=1\linewidth]{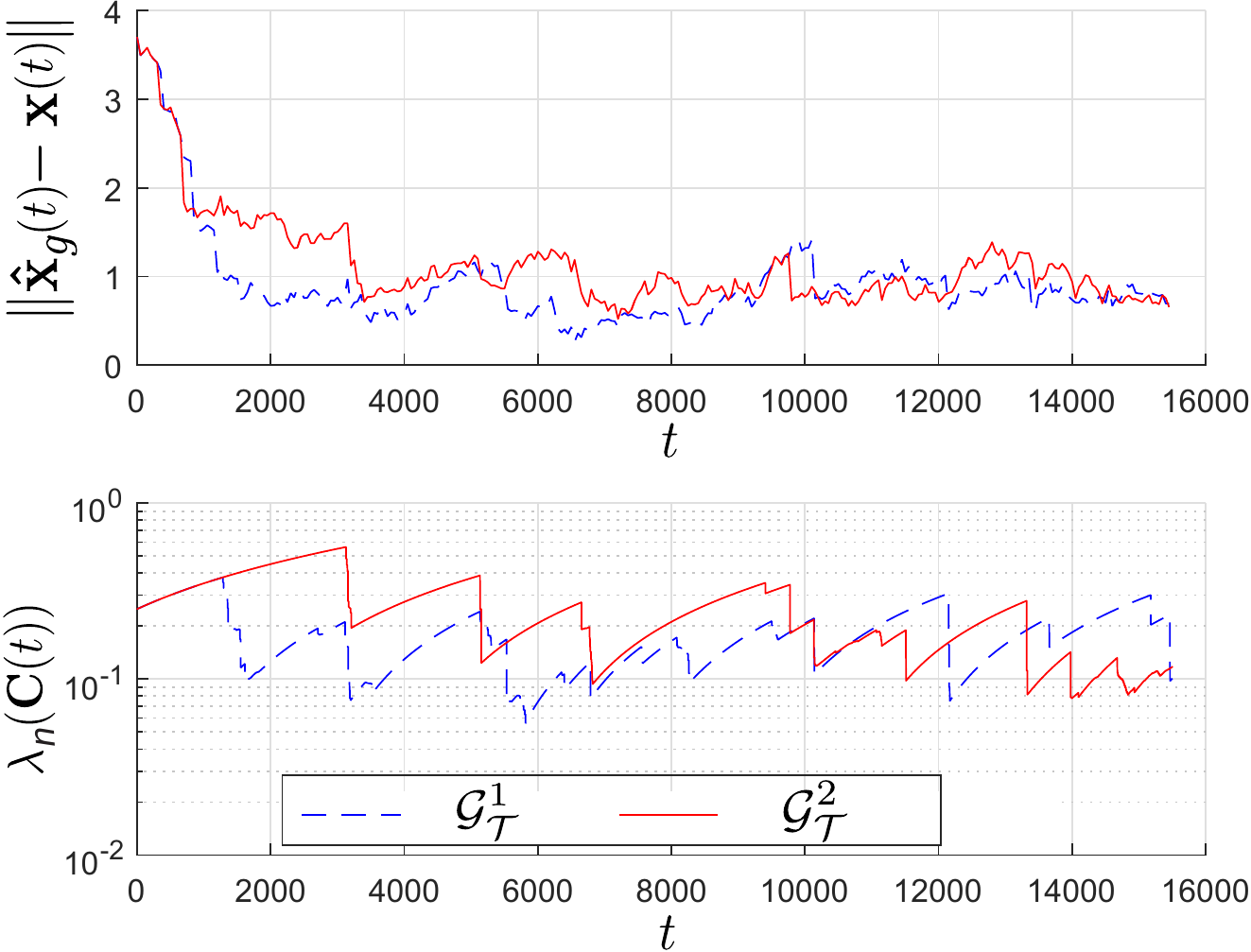}
\caption{Graphical depiction of the localization error $e_{\ell}(t)=\lVert \hat{\bbx}_{g}(t) - \bbx(t)\rVert$, and the uncertainty metric $\lambda_{n}(\bbC(t))$ with respect to time $t$ considering the graph topologies $\ccalG_{\ccalT}^1$ and $\ccalG_{\ccalT}^2$.}
\label{fig:compTop}
\end{figure}

The comparative results for these graphs are depicted in Figure \ref{fig:compTop}. Specifically, Figure \ref{fig:compTop} shows the evolution of the localization error $e_{\ell}(t)=\lVert \hat{\bbx}_{g}(t) - \bbx(t) \rVert$, and the uncertainty metric $\lambda_{n}(\bbC(t))$ as a function of discrete time $t$. In the error metric $e_{\ell}(t)$, $\bbx(t)$ is a vector that stacks the true positions of all targets and $\hat{\bbx}_{g}(t)$ stands for the estimate of $\bbx(t)$ and
corresponding covariance matrix $\bbC(t)$ assuming global fusion of all measurements taken by all robots at time $t$. 
Observe in Figure \ref{fig:compTop}, that the proposed algorithm can better estimate the positions of the targets, when the teams are constructed as in $\ccalG_{\ccalT}^1$. The reason is that the communication schedules $\texttt{sched}_{ij}$ in $\ccalG_{\ccalT}^1$ allow more teams to visit multiple targets at the same time compared to when the teams are determined according to $\ccalG_{\ccalT}^2$. Note that $\hat{\bbx}_g(t)$ is a fictitious estimate, as it is based on global fusion of measurements taken by all robots. However, as it will be discussed later, for any time instant $t$ there exists a time instant $t^*(t)<t$ up to which  the local estimate, computed using only local information available to team $\ccalT_i$, matches exactly the global estimate $\hat{\bbx}_g(t)$. Thus, the results in Figure \ref{fig:compTop} imply that the local estimates up to the time instant $t^*(t)$ are better when the teams are constructed as in $\ccalG_{\ccalT}^1$.
Define the average localization error and uncertainty as
\begin{subequations} \label{eq:avgErr}
\begin{align}
\bar{e}_{\ell} &= \frac{1}{t_{\text{end}}} \sum_{t=0}^{t_{\text{end}}}\lVert \hat{\bbx}_{g}(t) - \bbx(t) \rVert , \\
\bar{\lambda} &= \frac{1}{t_{\text{end}}} \sum_{t=0}^{t_{\text{end}}} \lambda_n(\bbC(t)),
\end{align}
\end{subequations}
where $t_{\text{end}}$ is the total number of discrete time steps $t$.
Then, when the teams are constructed as in $\ccalG_{\ccalT}^1$, $\bar{e}_{\ell} = 0.929$m and $\bar{\lambda} = 0.177$m$^2$, respectively. On the other hand,  when the teams are constructed as in $\ccalG_{\ccalT}^2$, the average error and average uncertainty are $\bar{e}_{\ell} = 1.127$m and $\bar{\lambda} = 0.532$m$^2$, respectively.

\begin{figure}[t]
\centering
\includegraphics[width=1\linewidth]{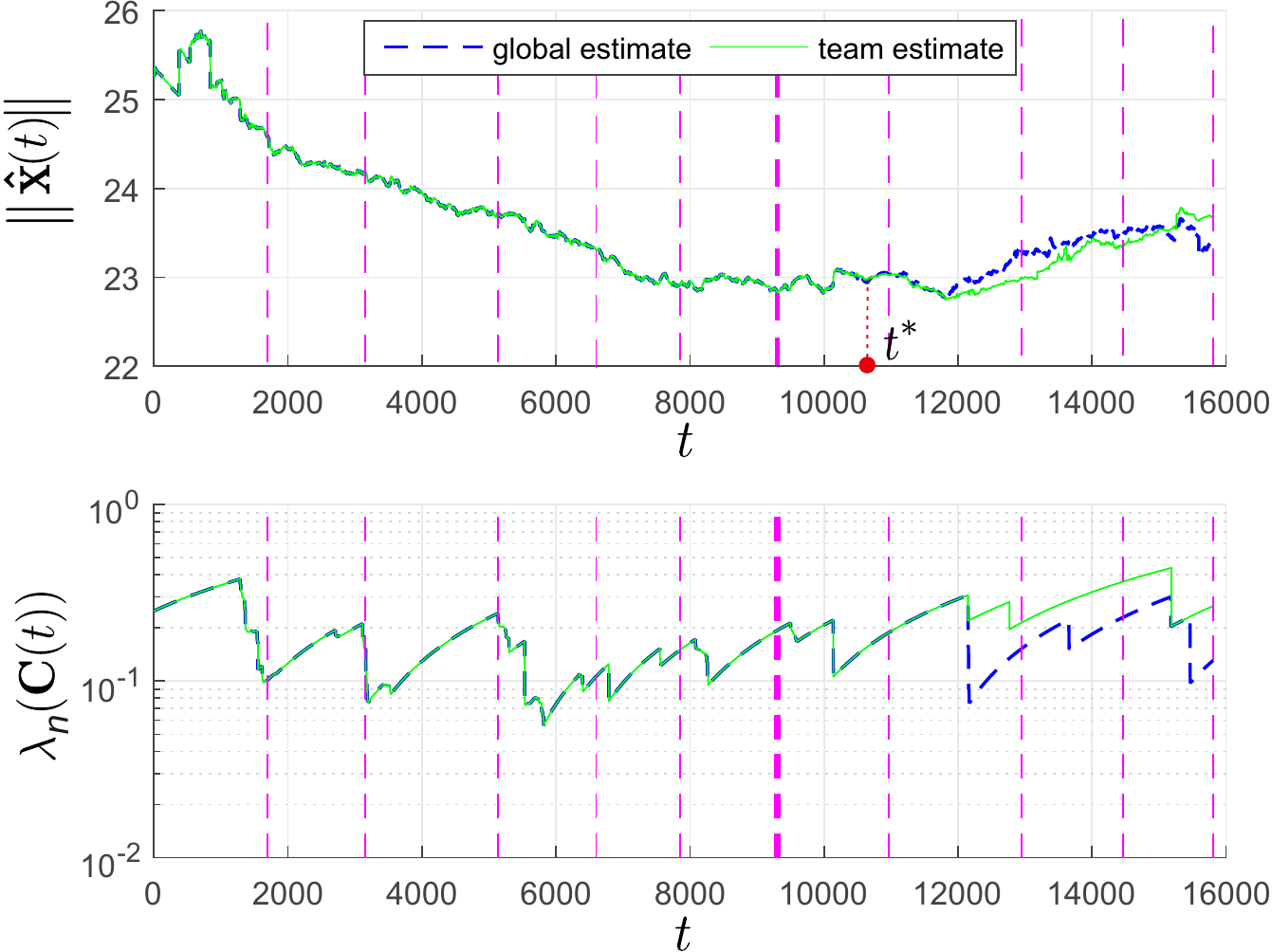}
\caption{Comparison of the evolution of the estimates $\hat{\bbx}_{g}(t)$, $\hat{\bbx}_{\ccalT_2}(t)$, and the uncertainty metric $\lambda_n(\bbC(t))$ when the teams are organized according to $\ccalG_{\ccalT}^1$. Each purple dashed line is associated with an epoch. Specifically, the $k$-th purple line corresponds to the time instant $t$ at which all robots in $\ccalT_2$ have reached the $k$-th epoch. The thick purple dashed line denotes the epoch up to which the evolution of the team estimate $\hat{\bbx}_{\ccalT_2}(t)$ matches exactly the evolution of $\hat{\bbx}_{g}(t)$ and corresponds to the lower bound provided in Proposition \ref{prop:kstar}.}
\label{fig:locFus}
\end{figure}

Next, we compare these two team topologies in terms of the delays they introduce in propagating data across teams. To this end, we define the following metric that encapsulates the average delay with which team $\ccalT_i$ receives information from all other teams:
\begin{equation}\label{eq:cons}
\bar{e}_d^i(t)=\frac{1}{t-t^*}\sum_{t^*}^t \lVert \hat{\bbx}_{g}(t)-\hat{\bbx}_{\ccalT_i}(t) \rVert. 
\end{equation}
In \eqref{eq:cons}, $t^*(t)$ stands for the time instant at which the global state estimate $\hat{\bbx}_{g}(t)$, and the local state estimate $\hat{\bbx}_{\ccalT_i}(t)$, after fusing measurements of team $\ccalT_i$, differ for the first time.\footnote{Note that the time instant $t^*$ depends on the considered team and epoch. Furthermore, this value does not necessarily coincide with any epoch time and is often larger than the lower bound provided in Proposition \ref{prop:kstar}.} 
For instance, this time instance $t^*(t)$ is illustrated in Figure \ref{fig:locFus} for team $\ccalT_2$ of $\ccalG_T^1$. Observe that the evolution of the estimate $\hat{\bbx}_{\ccalT_2}(t)$ matches exactly the evolution of $\hat{\bbx}_{g}(t)$ up to $4$ epochs earlier than the last epoch, as expected from Proposition \ref{prop:kstar}. Specifically, according to Proposition \ref{prop:kstar}, we have that $D_{\ccalG^1_\ccalT} = (2-1) \times 5 = 5$ which means that at any epoch $k$ all teams will fuse all measurements taken by the robots until epoch $k-5$. As a result, $\hat{\bbx}_{\ccalT_i}(t)$ at epoch $k$ should coincide with $\hat{\bbx}_{g}(t)$ at least up to epoch $k-5$.

Recall that the discrepancy between $\hat{\bbx}_{g}(t)$ and $\hat{\bbx}_{\ccalT_i}(t)$ is due to the fact that the network gets disconnected and, therefore, measurements taken by team $\ccalT_i$ are not communicated to team  $\ccalT_j$ instantaneously. The evolution of the error $\bbare_d^i(t)$ is depicted in Figure \ref{fig:compCons} for team $\ccalT_2$ for both team structures. Notice that $\bbare_d^2(t)$ is always larger for $\ccalG_T^1$, as expected due to Proposition \ref{prop:kstar}. Specifically, the delay for the graph $\ccalG_T^2$ is $D_{\ccalG^2_\ccalT} = (3-1) \times 2 = 4$ and is smaller than $D_{\ccalG^1_\ccalT} = 5$.
To further illustrate this observation, in Figure \ref{fig:compCons}, we also plot for team $\ccalT_2$ the value of $t^*$ for time instances at which communication within that team happens. Note that the time interval $t^*(t)-t$ is always larger when the teams are organized as in $\ccalG_T^1$. We conclude that the selection of the the graph $\ccalG_{\ccalT}$ depends on the specific problem and should achieve a balance between less uncertainty and smaller delays.

\begin{figure}[t]
\centering
\includegraphics[width=1\linewidth]{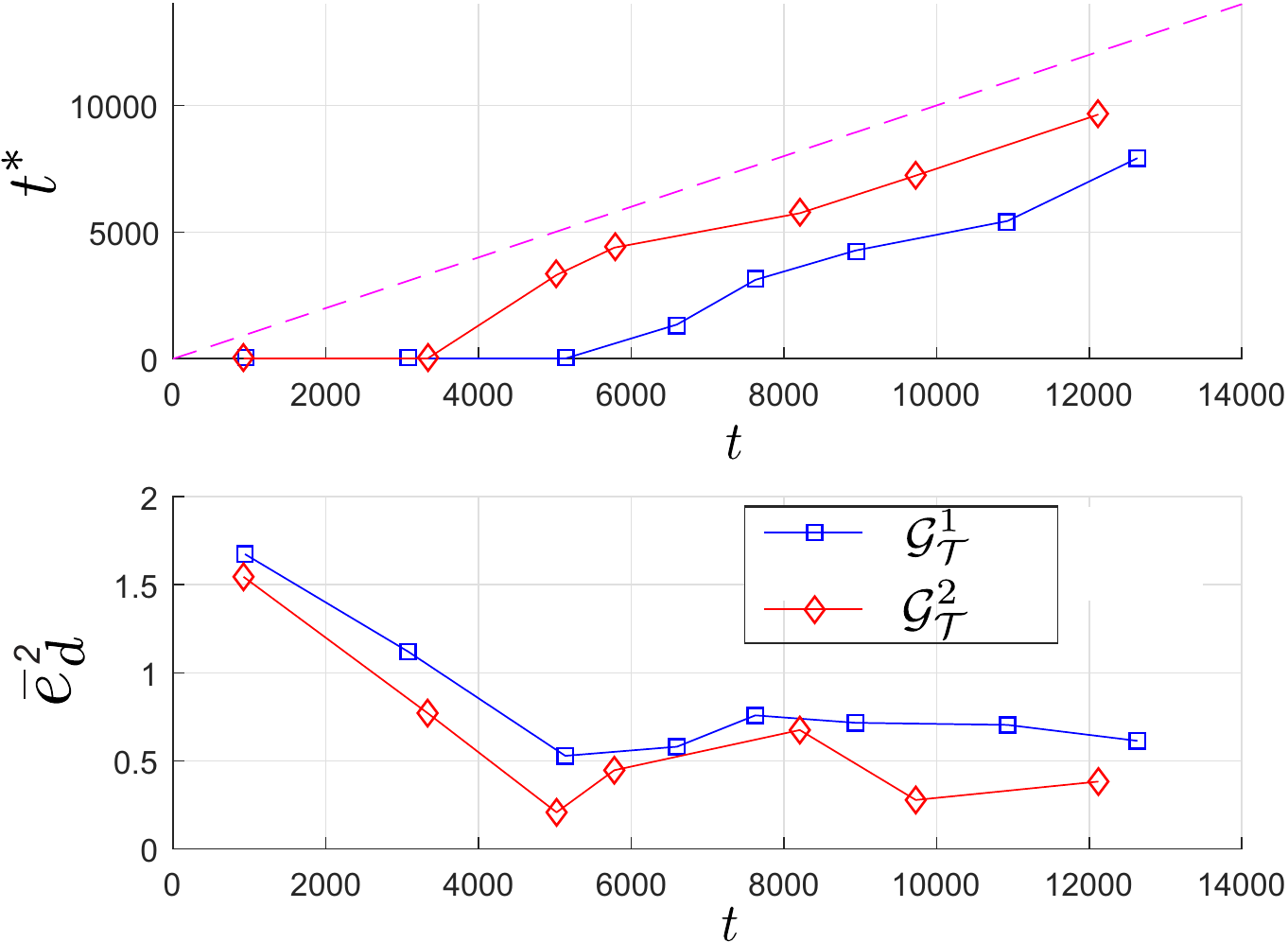}
\caption{Graphical depiction of the time instant $t^*(t)$ and the error $\bbare_d^2(t)$ as a function of time $t$ for the graphs $\ccalG_{\ccalT}^1$ and $\ccalG_{\ccalT}^2$. The red diamonds and blue squares correspond to communication events within team $\ccalT_2$ for the graphs $\ccalG_{\ccalT}^1$ and $\ccalG_{\ccalT}^2$, respectively. The dashed line in the top plot refers to the ideal case when $t^*(t)=t$, i.e., when there is no delay in communication. Thus, the closer the curve $t^*(t)$ is to the dashed line, the smaller the delays are.}
\label{fig:compCons}
\end{figure}

In Figure \ref{fig:paths}, we plot the paths that teams $\ccalT_4$ and $\ccalT_3$ in graph $\ccalG_{\ccalT}^1$ designed at epochs $k=4$ and  $k=5$, respectively. Recall, that $r_{34}\in\ccalT_3\cap\ccalT_4$. At $k=4$, robots of team $\ccalT_4$ design paths to decrease the uncertainty of target 3. However, observe in Figure \ref{fig:k4} that the robots, at the end of their paths, cannot sense target 3 due to an obstacle. Therefore, the uncertainty of that target will not decrease when the robots travel along these paths. The reason for this behavior is that the paths are designed based on the predicted state, given the filter initialization and available local information, and not the measurements. These paths will lead to a connected graph for team $\ccalT_4$ at epoch $4+T=6$. Observe also in Figure \ref{fig:k5} that the path for robot $r_{34}\in\ccalT_3\cap\ccalT_4$ starts at the end of its path designed at the previous epoch $k=4$, by construction of Algorithm \ref{alg:RRT}. A video simulation of the trajectories of the robots and targets for the graph $\ccalG_{\ccalT}^1$, and the sequence of meeting times and locations can be found in \cite{DSEvideo}.

\begin{figure}[t]
  \centering
     \subfigure[$k=4$]{
    \label{fig:k4}
  \includegraphics[width=0.85\linewidth]{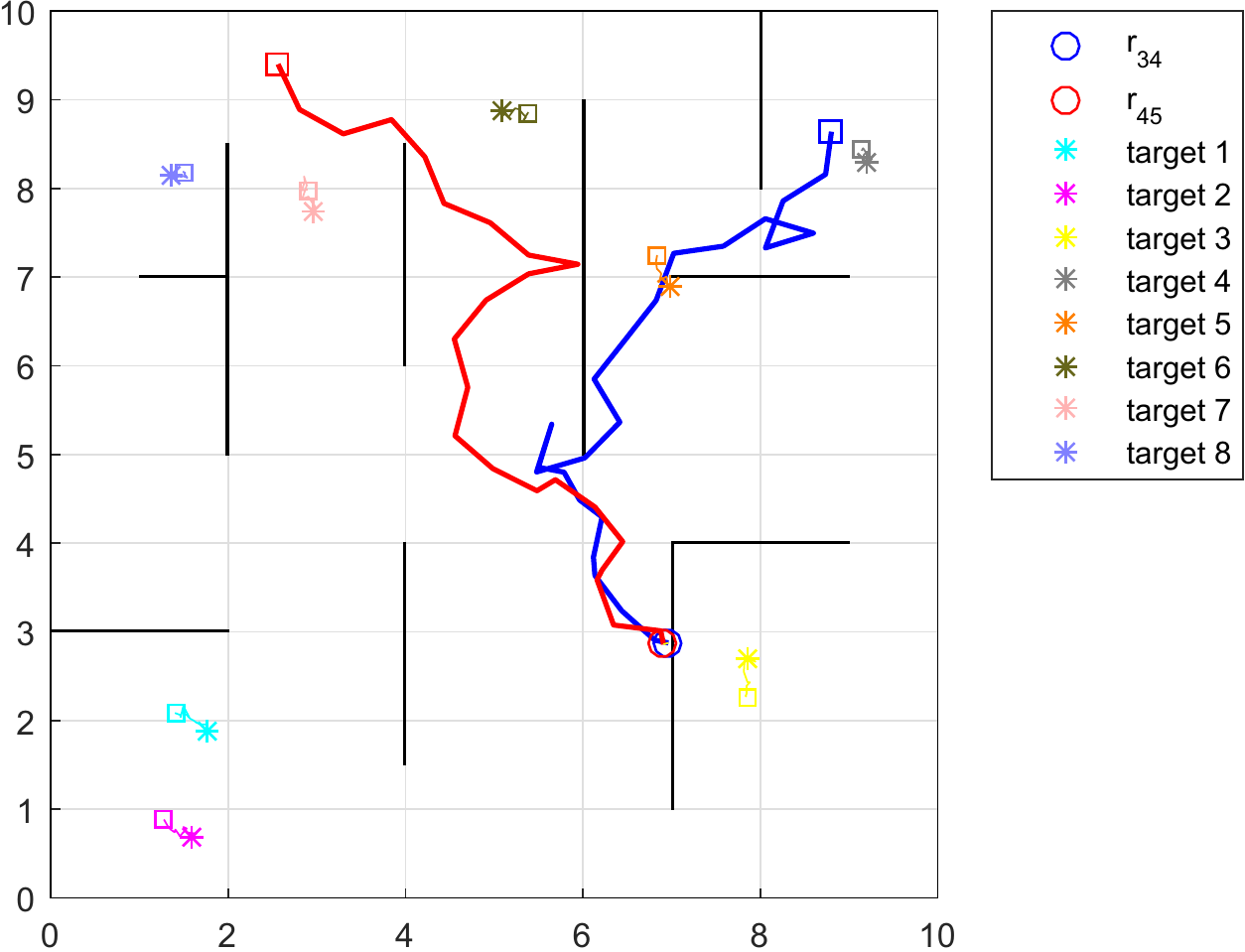}}
  \subfigure[$k=5$]{
    \label{fig:k5}
  \includegraphics[width=0.85\linewidth]{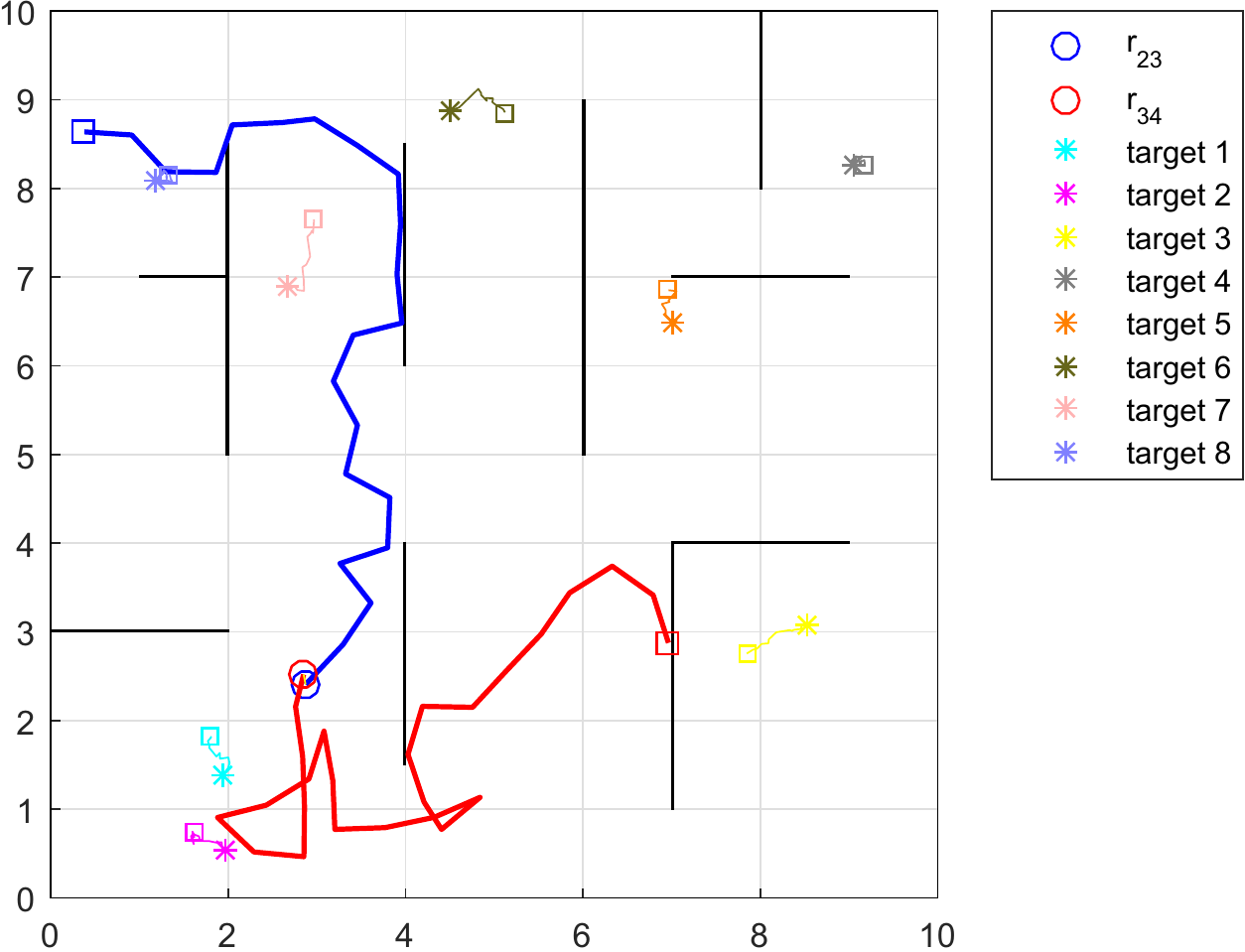}}
  \caption{Figures \ref{fig:k4} and \ref{fig:k5} depict the paths of teams $\ccalT_4$ and $\ccalT_3$, from $\ccalG_{\ccalT}^1$, designed at epochs $k=4$ and  $k=5$, and the corresponding paths of the targets during the time that the robots traverse these path segments. In these figures, the black straight lines show the obstacles and the squares indicate the beginning of the paths. Moreover, the circles and stars correspond to the end of the path segments for the robots and targets, respectively.}
  \label{fig:paths}
\end{figure} 

  
%


\subsection{Comparison with Alternative Approaches}\label{sim2}
In this section, we compare the proposed Algorithm \ref{alg:dhc} to an algorithm that preserves network connectivity for all time
and an intermittent connectivity heuristic that employs the same decomposition of robots into teams, but selects random meeting locations for each team. In what follows, we employ a team decomposition as in $\ccalG_{\ccalT}^1$. 
%
Specifically, for the connectivity preserving approach, we assume that the network of $N=8$ robots maintains a fixed connected configuration throughout the whole experiment. Then, to design informative paths for this network, we apply the $\text{RRT}^*$ algorithm to design paths for the geometric center of this configuration, so that the uncertainty of the most uncertain target drops below a threshold $\delta$. Once the connected network travels along the resulting path, the $\text{RRT}^*$ algorithm is executed again to find the next informative path. Notice that the paths constructed for the all-time connected network are asymptotically optimal.
%
On the other hand, for the heuristic approach the path segments $\bbp_{ij}^k$ are selected to be the geodesic paths that connect the initial locations to the randomly selected meeting location. The geodesic paths are computed using the toolbox in \cite{VisiLibity:08}.
When traveling along the paths designed by this heuristic, the time constraint is not satisfied and the robots need to wait at their meeting locations for their team members.

Figure \ref{fig:comp} shows the evolution of the localization error $e_{\ell}(t)=\lVert \hat{\bbx}_{g}(t) - \bbx(t)\rVert$, and the uncertainty metric $\lambda_{n}(\bbC(t))$ as a function of time $t$ for both approaches compared to our method. Notice that our method outperforms the algorithm that requires all robots to remain connected for all time. The reason is that our proposed algorithm allows the robots to disconnect in order to visit multiple targets simultaneously, which cannot happen using an all-time connectivity algorithm. Observe also that our algorithm performs better than the heuristic approach which justifies the use of a sampling-based algorithm to solve \eqref{eq:optimProb} instead of a computationally inexpensive method that picks random meeting locations connected to each other through geodesic paths.

The heuristic approach performs better in Figure \ref{fig:comp} than the algorithm that enforces end-to-end network connectivity for all time; this is not always the case.
To elaborate more, we compare the performance of the algorithms for $4$ and $8$ targets in Table \ref{tab:comp} in terms of the average localization error and the average uncertainty defined in \eqref{eq:avgErr}. 
Observe that as we decrease the number of targets, the all-time connectivity algorithm performs better than the heuristic approach. The reason for this is that the all-time connectivity approach forces the network of robots to visit the targets sequentially. Thus, for larger number of targets it takes longer to revisit a specific target and this results in uncertainty spikes. On the other hand, the heuristic approach selects the meeting locations randomly. Thus, as we populate the domain with targets, the paths designed by the heuristic cross nearby a target with larger probability which improves its performance. Our proposed algorithm outperforms both approaches regardless of the number of targets. In Table \ref{tab:comp}, for the case of 4 targets, we considered a network with $N=4$ robots divided in the following teams $\ccalT_1=\set{r_{12}, r_{14}}$, $\ccalT_2=\set{r_{12}, r_{23}}$, $\ccalT_3=\set{r_{23}, r_{34}}$, and $\ccalT_4=\set{r_{34}, r_{14}}$.


\begin{figure}[t]
\centering
\includegraphics[width=1\linewidth]{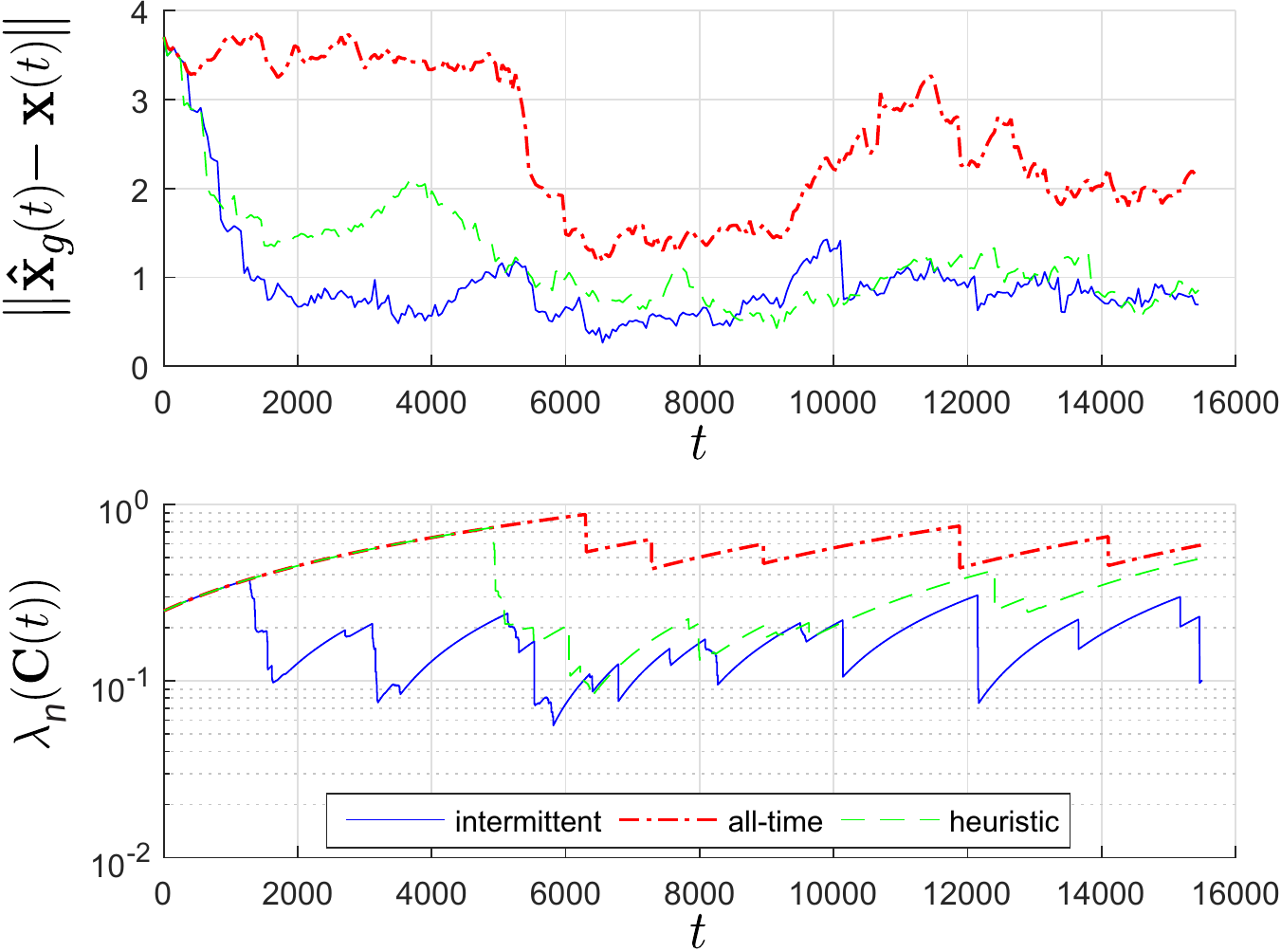}
\caption{Comparison of the evolution of the localization error $e_{\ell}(t)=\lVert \hat{\bbx}_{g}(t)- \hat{\bbx}(t) \rVert$ and the uncertainty metric $\lambda_n(\bbC(t))$ considering teams organized as per team $\ccalG_{\ccalT}^1$, between our proposed control framework, an all-time connectivity algorithm, and a heuristic approach.}
\label{fig:comp}
\end{figure}

%
%

\begin{table*}[t]
\caption{Comparison with Alternative Approaches}
\label{tab:comp}
\begin{center}
\renewcommand{\arraystretch}{1.5}
\begin{tabular}{c|c|c|c||c|c|c|}
	\cline{2-7}
    	\multirow{2}{*}{}		&	\multicolumn{3}{c||}{$4$ Targets}						&	\multicolumn{3}{c|}{$8$ Targets}					\\ \cline{2-7}
					&	all-time 		&	heuristic		& 	intermittent		&	all-time 		&	heuristic		& 	intermittent 	\\  \hline
					
	\multicolumn{1}{|c||}{$\bar{e}_{\ell}$ (m)}		&
					1.030 $\pm$ 0.466	& 1.993 $\pm$ 0.566	& 0.879 $\pm$ 0.461		& 2.517 $\pm$ 0.810 & 1.201 $\pm$ 0.570 & 0.929 $\pm$ 0.561 \\ \hline
						
	\multicolumn{1}{|c||}{$\bar{\lambda}$ (m$^2$) }			&
					0.226 $\pm$ 0.061	& 0.405 $\pm$ 0.181	& 0.167 $\pm$ 0.084		& 0.563 $\pm$ 0.129	 & 0.334 $\pm$ 0.161 & 0.177 $\pm$ 0.068 \\ \hline
\end{tabular}
\end{center}
\end{table*}

Finally observe that the local estimation results for team $\ccalT_2$ in Figure \ref{fig:locFus} are still considerably better than the results of competitive methods shown in Figure \ref{fig:comp}. Particularly, for $8$ targets the numerical values for $\ccalT_2$ are $\bar{e}_{\ell} =1.044 \pm 0.628$m and $\bar{\lambda} = 0.177 \pm 0.071$m$^2$, respectively; compare with Table \ref{tab:comp}. This is important since in practice, a user can have access to the information of an individual robot and not the whole network, as this would defeat the purpose of assuming robots with limited communication range. In a similar scenario in the case of all-time connectivity,  the whole network needs to be recalled to an access point to collect information from the robots, which would disrupt the estimation task.

\section{Conclusion} \label{sec:Concl}
In this paper, we considered the problem of distributed state estimation using intermittently connected mobile robot networks. Robots were assumed to have limited communication capabilities and, therefore, they exchanged their information intermittently only when they were sufficiently close to each other. To the best of our knowledge, this is the first framework for concurrent state estimation and intermittent communication control that does not rely on any network connectivity assumption and can be applied to large-scale networks. We presented simulation results that demonstrated significant improvement in estimation accuracy compared to methods that maintain an end-to-end connected network for all time. The improvement becomes more considerable as information becomes localized and sparse and the communication and sensing ranges of the robots become smaller compared to the size of the domain.

\appendices
\section{Steering function}\label{sec:steer}

In this section we discuss the details of the steering function, explained in Section \ref{sec:steer1}, for robots with single-integrator dynamics. Define $\texttt{delay(\bbv)}=\max_{r_{ij}\in\ccalT_i}t_{ij}(\bbv)-\min_{r_{ij}\in\ccalT_i}t_{ij}(\bbv)$.
Then, the goal of the steering function described in Algorithm \ref{alg:steer} is to match the time instants $t_{ij}(\bbv_{\text{new}})$ that the robots $r_{ij}\in\ccalT_i$ arrive at $\bbv_{\text{new}}$ so that $\texttt{delay}(\bbv_{\text{new}})=0$. 
This can be accomplished by adjusting the velocities or the step-sizes of the robots. Since we are interested in estimation of dynamic fields, allowing the robots to move with maximum speed results in more optimal measurements. Thus, here we adjust the step-sizes to match the times of the robots. Once the times are matched, the robots take steps with equal length to maintain zero delay. To this end, in [line \ref{steer:line1}, Alg. \ref{alg:steer}] we project the vector $\bbv_{\text{rand}}-\bbv_{\text{nearest}}$ onto the space of every robot $r_{ij}\in\ccalT_i$. We denote this projection by $\bbd_{ij}$. 
Let $\bbarepsilon>0$ be a user-specified parameter that determines the maximum step-size of the robots.
Then, in [line \ref{steer:line2}, Alg. \ref{alg:steer}] we compute the smallest step size $s$ of the robots as $s=\min_{r_{ij}\in\ccalT_i} \left( \bbarepsilon , ||\bbd_{ij} || \right)$. Given the step size $s$, in [line \ref{steer:line3}, Alg. \ref{alg:steer}] we compute the time interval $\Delta t_s$ required to travel distance $s$ with maximum velocity $u_{\max}$, i.e., $\Delta t_s = {s}/{ u_{\max} }$. 
Then, in [line \ref{steer:line5}, Alg. \ref{alg:steer}] we compute $\Delta t_{c,ij} = \Delta t_s- (t_{ij}(\bbv_{\text{nearest}})-\min_{r_{ij}\in\ccalT_i}t_{ij}(\bbv_{\text{nearest}}))$.
This quantity captures the expected delay of robot $r_{ij}$ at the new sample.
If $\Delta t_{c,ij}>0$, then the projection of $\bbv_\text{new}$ to the space of robot $r_{ij}$, denoted by $\bbv_{\text{new},ij}$, is computed as $\bbv_{\text{new},ij} = \bbv_{\text{nearest},ij} + u_{\max}  \Delta t_{c,ij} {\bbd_{ij}} / {\lVert \bbd_{ij}\rVert} $ and $t_{ij}(\bbv_{\text{new}})=t_{ij}(\bbv_{\text{nearest}})+\Delta t_{c,ij} $. Note that the robot $r_{ij}$ associated with the time instant $\min_{r_{ij}\in\ccalT_i}t_{ij}(\bbv_{\text{nearest}})$ will always move with step size of $s$, as it trivially satisfies $\Delta t_{c,ij}>0$. If $\Delta t_{c,ij}\leq0$, the robot $r_{ij}$ does not move, i.e., $\bbv_{\text{new},ij} = \bbv_{\text{nearest},ij}$ and  $t_{ij}(\bbv_{\text{new}})=t_{ij}(\bbv_{\text{nearest}})$ [lines \ref{steer:line6}-\ref{steer:line9}, Alg. \ref{alg:steer}]. 
This ensures that $\texttt{delay}(\bbv_{\text{new}})\leq\texttt{delay}(\bbv_{\text{nearest}})$ and, as the tree grows and the length of its branches increases, nodes $\bbv\in\Omega_{\text{free}}^{|\ccalT_i|}$ will eventually be added to the tree so that $\texttt{delay}(\bbv)=0$. 

\begin{algorithm}[t] 
\caption{\texttt{Steer}}
\label{alg:steer}
\begin{algorithmic}[1]
\REQUIRE Nodes $\bbv_{\text{nearest}}$ and $\bbv_{\text{rand}}$;
\STATE Compute vector $\bbd_{ij}$ by projecting $\bbv_{\text{rand}}-\bbv_{\text{nearest}}$ onto the space of  robot $r_{ij}\in\ccalT_i$; \label{steer:line1}
\STATE Compute $s=\min_{r_{ij}\in\ccalT_i} \left( \bbarepsilon , ||\bbd_{ij} || \right)$; \label{steer:line2}
\STATE Compute travel time $\Delta t_s = {s}/{ u_{\max} }$; \label{steer:line3}
\STATE Compute the expected delay for each robot: \\$\Delta t_{c,ij}=\Delta t_s- (t_{ij}(\bbv_{\text{nearest}})-\min_{r_{ij}\in\ccalT_i}t_{ij}(\bbv_{\text{nearest}}))$; \label{steer:line5}

		\IF {$\Delta t_{c,ij} > 0$}	\label{steer:line6}
			\STATE $\bbv_{\text{new},ij}=\bbv_{\text{nearest},ij} + u_{\max} \Delta t_{c,ij} {\bbd_{ij}}/{\lVert \bbd_{ij}\rVert} $;	\label{steer:line7}
	    \ELSE
			\STATE $\bbv_{\text{new},ij}=\bbv_{\text{nearest},ij}$; \hspace{2mm} \small{\texttt{\% robot $r_{ij}$ does not move}}	\label{steer:line9}
		\ENDIF

\end{algorithmic}
\end{algorithm} 

\bibliographystyle{IEEEtran}
\bibliography{MyBibliography,YK_bib}

\begin{thebibliography}{10}
\providecommand{\url}[1]{#1}
\csname url@samestyle\endcsname
\providecommand{\newblock}{\relax}
\providecommand{\bibinfo}[2]{#2}
\providecommand{\BIBentrySTDinterwordspacing}{\spaceskip=0pt\relax}
\providecommand{\BIBentryALTinterwordstretchfactor}{4}
\providecommand{\BIBentryALTinterwordspacing}{\spaceskip=\fontdimen2\font plus
\BIBentryALTinterwordstretchfactor\fontdimen3\font minus
  \fontdimen4\font\relax}
\providecommand{\BIBforeignlanguage}[2]{{%
\expandafter\ifx\csname l@#1\endcsname\relax
\typeout{** WARNING: IEEEtran.bst: No hyphenation pattern has been}%
\typeout{** loaded for the language `#1'. Using the pattern for}%
\typeout{** the default language instead.}%
\else
\language=\csname l@#1\endcsname
\fi
#2}}
\providecommand{\BIBdecl}{\relax}
\BIBdecl

\bibitem{PACML2000FBKT}
D.~Fox, W.~Burgard, H.~Kruppa, and S.~Thrun, ``A probabilistic approach to
  collaborative multi-robot localization,'' \emph{Autonomous Robots}, vol.~8,
  no.~3, pp. 325--344, 2000.

\bibitem{DML2002RB}
S.~I. Roumeliotis and G.~A. Bekey, ``Distributed multirobot localization,''
  \emph{{IEEE} Transactions on Robotics and Automation}, vol.~18, no.~5, pp.
  781--795, 2002.

\bibitem{SLAM2006DB}
H.~Durrant-Whyte and T.~Bailey, ``Simultaneous localization and mapping: Part
  i,'' \emph{{IEEE} Transactions on Robotics and Automation}, vol.~13, no.~2,
  pp. 99--110, 2006.

\bibitem{nerurkar2014c}
E.~D. Nerurkar, K.~J. Wu, and S.~I. Roumeliotis, ``C-klam: Constrained
  keyframe-based localization and mapping,'' in \emph{IEEE International
  Conference on Robotics and Automation}, Hong Kong, China, May-June 2014, pp.
  3638--3643.

\bibitem{DAIA2015ALDP}
N.~Atanasov, J.~Le~Ny, K.~Daniilidis, and G.~J. Pappas, ``Decentralized active
  information acquisition: theory and application to multi-robot slam,'' in
  \emph{Proceedings of {IEEE} International Conference on Robotics and
  Automation}.\hskip 1em plus 0.5em minus 0.4em\relax IEEE, 2015, pp.
  4775--4782.

\bibitem{leung2012decentralized}
K.~Y. Leung, T.~D. Barfoot, and H.~H. Liu, ``Decentralized cooperative slam for
  sparsely-communicating robot networks: A centralized-equivalent approach,''
  \emph{Journal of Intelligent \& Robotic Systems}, vol.~66, no.~3, pp.
  321--342, 2012.

\bibitem{ISDSCMSN2012JO}
P.~Jalalkamali and R.~Olfati-Saber, ``Information-driven self-deployment and
  dynamic sensor coverage for mobile sensor networks,'' in \emph{Proceedings of
  American Control Conference}.\hskip 1em plus 0.5em minus 0.4em\relax IEEE,
  2012, pp. 4933--4938.

\bibitem{vander2015algorithms}
J.~Vander~Hook, P.~Tokekar, and V.~Isler, ``Algorithms for cooperative active
  localization of static targets with mobile bearing sensors under
  communication constraints,'' \emph{IEEE Transactions on Robotics}, vol.~31,
  no.~4, pp. 864--876, 2015.

\bibitem{OPPRAATL2015FMZ}
C.~Freundlich, P.~Mordohai, and M.~M. Zavlanos, ``Optimal path planning and
  resource allocation for active target localization,'' in \emph{Proceedings of
  American Control Conference}.\hskip 1em plus 0.5em minus 0.4em\relax IEEE,
  2015, pp. 3088--3093.

\bibitem{freundlich2017distributed}
C.~Freundlich, Y.~Zhang, and M.~M. Zavlanos, ``Distributed hierarchical control
  for state estimation with robotic sensor networks,'' \emph{IEEE Transactions
  on Control of Network Systems}, vol.~PP, no.~99, pp. 1--1, December 2017.

\bibitem{OSPMCTT2006MB}
S.~Mart{\'\i}nez and F.~Bullo, ``Optimal sensor placement and motion
  coordination for target tracking,'' \emph{Automatica}, vol.~42, no.~4, pp.
  661--668, 2006.

\bibitem{hollinger2009efficient}
G.~Hollinger, S.~Singh, J.~Djugash, and A.~Kehagias, ``Efficient multi-robot
  search for a moving target,'' \emph{The International Journal of Robotics
  Research}, vol.~28, no.~2, pp. 201--219, 2009.

\bibitem{DMCSDTT2006CBM}
T.~H. Chung, J.~W. Burdick, and R.~M. Murray, ``A decentralized motion
  coordination strategy for dynamic target tracking,'' in \emph{Proceedings of
  {IEEE} International Conference on Robotics and Automation}.\hskip 1em plus
  0.5em minus 0.4em\relax IEEE, 2006, pp. 2416--2422.

\bibitem{OACTTPG2009DSH}
J.~Derenick, J.~Spletzer, and A.~Hsieh, ``An optimal approach to collaborative
  target tracking with performance guarantees,'' \emph{Journal of Intelligent
  \& Robotic Systems}, vol.~56, no.~1, pp. 47--67, 2009.

\bibitem{DTMSNIM2007O}
R.~Olfati-Saber, ``Distributed tracking for mobile sensor networks with
  information-driven mobility,'' in \emph{Proceedings of American Control
  Conference}.\hskip 1em plus 0.5em minus 0.4em\relax IEEE, 2007, pp.
  4606--4612.

\bibitem{huang2015bank}
G.~Huang, K.~Zhou, N.~Trawny, and S.~I. Roumeliotis, ``A bank of maximum a
  posteriori (map) estimators for target tracking,'' \emph{IEEE Transactions on
  Robotics}, vol.~31, no.~1, pp. 85--103, 2015.

\bibitem{kantaros2017temporal}
\BIBentryALTinterwordspacing
Y.~Kantaros, M.~Guo, and M.~M. Zavlanos, ``Temporal logic task planning and
  intermittent connectivity control of mobile robot networks,'' \emph{IEEE
  Transactions on Automatic Control}, (accepted). [Online]. Available:
  \url{https://arxiv.org/abs/1706.00765}
\BIBentrySTDinterwordspacing

\bibitem{freundlich2018distributed}
C.~Freundlich, S.~Lee, and M.~M. Zavlanos, ``Distributed active state
  estimation with user-specified accuracy,'' \emph{IEEE Transactions on
  Automatic Control}, vol.~63, no.~2, pp. 418--433, 2018.

\bibitem{DKFSN2007O}
R.~Olfati-Saber, ``Distributed kalman filtering for sensor networks,'' in
  \emph{Proceedings of {IEEE} Conference on Decision and Control}.\hskip 1em
  plus 0.5em minus 0.4em\relax IEEE, 2007, pp. 5492--5498.

\bibitem{IWC2012KFR}
A.~T. Kamal, J.~A. Farrell, and A.~K. Roy-Chowdhury, ``Information weighted
  consensus,'' in \emph{Proceedings of {IEEE} Conference on Decision and
  Control}.\hskip 1em plus 0.5em minus 0.4em\relax IEEE, 2012, pp. 2732--2737.

\bibitem{DRSSIA2012JASR}
B.~J. Julian, M.~Angermann, M.~Schwager, and D.~Rus, ``Distributed robotic
  sensor networks: An information-theoretic approach,'' \emph{International
  Journal of Robotics Research}, vol.~31, no.~10, pp. 1134--1154, 2012.

\bibitem{DDFCASN2004MD}
A.~Makarenko and H.~Durrant-Whyte, ``Decentralized data fusion and control in
  active sensor networks,'' in \emph{International Conference on Information
  Fusion}, vol.~1, 2004, pp. 479--486.

\bibitem{DDF2016CA}
M.~E. Campbell and N.~R. Ahmed, ``Distributed data fusion: Neighbors, rumors,
  and the art of collective knowledge,'' \emph{IEEE Control Systems}, vol.~36,
  no.~4, pp. 83--109, 2016.

\bibitem{DADDSN2001ND}
E.~W. Nettleton and H.~F. Durrant-Whyte, ``Delayed and asequent data in
  decentralized sensing networks,'' in \emph{Intelligent Systems and Advanced
  Manufacturing}.\hskip 1em plus 0.5em minus 0.4em\relax International Society
  for Optics and Photonics, 2001, pp. 1--9.

\bibitem{KFIO2004SSFPJS}
B.~Sinopoli, L.~Schenato, M.~Franceschetti, K.~Poolla, M.~I. Jordan, and S.~S.
  Sastry, ``Kalman filtering with intermittent observations,'' \emph{{IEEE}
  Transactions on Automatic Control}, vol.~49, no.~9, pp. 1453--1464, 2004.

\bibitem{SDCMRIGT2006MD}
G.~M. Mathews and H.~F. Durrant-Whyte, ``Scalable decentralised control for
  multi-platform reconnaissance and information gathering tasks,'' in
  \emph{International Conference on Information Fusion}.\hskip 1em plus 0.5em
  minus 0.4em\relax IEEE, 2006, pp. 1--8.

\bibitem{lan2016rapidly}
X.~Lan and M.~Schwager, ``Rapidly exploring random cycles: Persistent
  estimation of spatiotemporal fields with multiple sensing robots,''
  \emph{IEEE Transactions on Robotics}, vol.~32, no.~5, pp. 1230--1244, 2016.

\bibitem{hollinger2013sampling}
G.~A. Hollinger and G.~S. Sukhatme, ``Sampling-based motion planning for
  robotic information gathering.'' in \emph{Robotics: Science and
  Systems}.\hskip 1em plus 0.5em minus 0.4em\relax Citeseer, 2013, pp. 72--983.

\bibitem{kuffner2000rrt}
J.~J. Kuffner and S.~M. LaValle, ``Rrt-connect: An efficient approach to
  single-query path planning,'' in \emph{IEEE International Conference on
  Robotics and Automation}, San Fransisco, CA, April 2000, pp. 995--1001.

\bibitem{karaman2011sampling}
S.~Karaman and E.~Frazzoli, ``Sampling-based algorithms for optimal motion
  planning,'' \emph{The International Journal of Robotics Research}, vol.~30,
  no.~7, pp. 846--894, 2011.

\bibitem{hollinger2014sampling}
G.~A. Hollinger and G.~S. Sukhatme, ``Sampling-based robotic information
  gathering algorithms,'' \emph{The International Journal of Robotics
  Research}, vol.~33, no.~9, pp. 1271--1287, 2014.

\bibitem{kim2006maximizing}
Y.~Kim and M.~Mesbahi, ``On maximizing the second smallest eigenvalue of a
  state-dependent graph laplacian,'' \emph{IEEE Transactions on Automatic
  Control}, vol.~51, no.~1, pp. 116--120, 2006.

\bibitem{zavlanos2007potential}
M.~M. Zavlanos and G.~J. Pappas, ``Potential fields for maintaining
  connectivity of mobile networks,'' \emph{IEEE Transactions on Robotics,},
  vol.~23, no.~4, pp. 812--816, 2007.

\bibitem{ji2007distributed}
M.~Ji and M.~B. Egerstedt, ``Distributed coordination control of multi-agent
  systems while preserving connectedness.'' \emph{IEEE Transactions on
  Robotics}, vol.~23, no.~4, pp. 693--703, August 2007.

\bibitem{Zavlanos_IEEETRO08}
M.~Zavlanos and G.~Pappas, ``Distributed connectivity control of mobile
  networks,'' \emph{IEEE Transactions on Robotics}, vol.~24, no.~6, pp.
  1416--1428, 2008.

\bibitem{yang2010decentralized}
P.~Yang, R.~A. Freeman, G.~J. Gordon, K.~M. Lynch, S.~S. Srinivasa, and
  R.~Sukthankar, ``Decentralized estimation and control of graph connectivity
  for mobile sensor networks,'' \emph{Automatica}, vol.~46, no.~2, pp.
  390--396, 2010.

\bibitem{montijano2011adaptive}
E.~Montijano, J.~Montijano, and C.~Sagues, ``Adaptive consensus and algebraic
  connectivity estimation in sensor networks with chebyshev polynomials,'' in
  \emph{50th IEEE Conference on Decision and Control and European Control
  Conference}, Orlando, FL, USA, December 2011, pp. 4296--4301.

\bibitem{franceschelli2013decentralized}
M.~Franceschelli, A.~Gasparri, A.~Giua, and C.~Seatzu, ``Decentralized
  estimation of laplacian eigenvalues in multi-agent systems,''
  \emph{Automatica}, vol.~49, no.~4, pp. 1031--1036, 2013.

\bibitem{sabattini2013decentralized}
L.~Sabattini, N.~Chopra, and C.~Secchi, ``Decentralized connectivity
  maintenance for cooperative control of mobile robotic systems,'' \emph{The
  International Journal of Robotics Research}, vol.~32, no.~12, pp. 1411--1423,
  2013.

\bibitem{Zavlanos_IEEE11}
M.~Zavlanos, M.~Egerstedt, and G.~Pappas, ``Graph theoretic connectivity
  control of mobile robot networks,'' \emph{Proceedings of the IEEE: Special
  Issue on Swarming in Natural and Engineered Systems}, vol.~99, no.~9, pp.
  1525--1540, 2011.

\bibitem{FLSPPDF2016CF}
A.~J. Carfang and E.~W. Frew, ``Fast link scheduling policies for persistent
  data ferrying,'' \emph{Journal of Aerospace Information Systems}, pp.
  433--449, 2016.

\bibitem{hollinger2010multi}
G.~Hollinger and S.~Singh, ``Multi-robot coordination with periodic
  connectivity,'' in \emph{IEEE International Conference on Robotics and
  Automation}, Anchorage, Alaska, May 2010, pp. 4457--4462.

\bibitem{zavlanos2010synchronous}
M.~M. Zavlanos, ``Synchronous rendezvous of very-low-range wireless agents,''
  in \emph{49th IEEE Conference on Decision and Control}, Atlanta, GA, USA,
  December 2010, pp. 4740--4745.

\bibitem{kantaros2016distributedInterm}
Y.~Kantaros and M.~M. Zavlanos, ``Distributed intermittent connectivity control
  of mobile robot networks,'' \emph{Transactions on Automatic Control},
  vol.~62, no.~7, pp. 3109--3121, July 2017.

\bibitem{kantaros2016simultaneous}
------, ``Simultaneous intermittent communication control and path optimization
  in networks of mobile robots,'' in \emph{55th IEEE Conference on Decision and
  Control}, Las Vegas, NV, December 2016, pp. 1794--1795.

\bibitem{kantaros2018distributed}
Y.~Kantaros and M.~Zavlanos, ``Distributed intermittent communication control
  of mobile robot networks under time-critical dynamic tasks,'' in \emph{IEEE
  International Conference on Robotics and Automation}, Brisbane, Australia,
  May 2018, pp. 5028--5033.

\bibitem{guo2018gathering}
M.~Guo and M.~M. Zavlanos, ``Multi-robot data gathering under buffer
  constraints and intermittent communication,'' \emph{IEEE Transactions on
  Robotics}, no.~4, pp. 1082--1097, 2018.

\bibitem{meJ1}
R.~Khodayi-mehr, W.~Aquino, and M.~M. Zavlanos, ``Model-based active source
  identification in complex environments,'' \emph{{IEEE} Transactions on
  Robotics}, (accepted). [Online]. Available:
  https://arxiv.org/pdf/1706.01603.pdf, 2018.

\bibitem{leahy2015distributed}
K.~Leahy, A.~Jones, M.~Schwager, and C.~Belta, ``Distributed information
  gathering policies under temporal logic constraints,'' in \emph{IEEE 54th
  Conference on Decision and Control}, Osaka, Japan, 2015, pp. 6803--6808.

\bibitem{Godsil_SPRINGER01}
C.~Godsil and G.~Royle, \emph{Algebraic Graph Theory}.\hskip 1em plus 0.5em
  minus 0.4em\relax New York: Springer-Verlag, 2001.

\bibitem{karaman2010optimal}
S.~Karaman and E.~Frazzoli, ``Optimal kinodynamic motion planning using
  incremental sampling-based methods,'' in \emph{49th IEEE Conference on
  Decision and Control}, Atlanta, GA, December 2010, pp. 7681--7687.

\bibitem{DFDSN2001DS}
H.~Durrant-Whyte and M.~Stevens, ``Data fusion in decentralised sensing
  networks,'' in \emph{International Conference on Information Fusion}, 2001,
  pp. 302--307.

\bibitem{CSKFNSE2015LKGD}
J.~Y. Li, A.~Kokkinaki, H.~Ghorbanidehno, E.~F. Darve, and P.~K. Kitanidis,
  ``The compressed state kalman filter for nonlinear state estimation:
  Application to large-scale reservoir monitoring,'' \emph{Water Resources
  Research}, vol.~51, no.~12, pp. 9942--9963, 2015.

\bibitem{RTDALSS2015GKLD}
H.~Ghorbanidehno, A.~Kokkinaki, J.~Y. Li, E.~Darve, and P.~K. Kitanidis,
  ``Real-time data assimilation for large-scale systems: The spectral kalman
  filter,'' \emph{Advances in water resources}, vol.~86, pp. 260--272, 2015.

\bibitem{DDFMN2005MC}
T.~W. Martin and K.-C. Chang, ``A distributed data fusion approach for mobile
  ad hoc networks,'' in \emph{International Conference on Information Fusion},
  vol.~2.\hskip 1em plus 0.5em minus 0.4em\relax IEEE, 2005, pp. 8--pp.

\bibitem{meC1}
R.~Khodayi-mehr, W.~Aquino, and M.~M. Zavlanos, ``Model-based sparse source
  identification,'' in \emph{Proceedings of American Control Conference}, July
  2015, pp. 1818--1823.

\bibitem{meC4}
------, ``Distributed reduced order source identification,'' in
  \emph{Proceedings of American Control Conference}, June 2018, pp. 1084--1089.

\bibitem{muraca2009state}
P.~Muraca and P.~Pugliese, ``State estimation for distributed parameter systems
  under intermittent observations,'' in \emph{IEEE International Conference on
  Control and Automation}, Christchurch, New Zealand, 2009, pp. 926--930.

\bibitem{DSEvideo}
``Simulation video - \url{https://vimeo.com/262568673}.''

\bibitem{VisiLibity:08}
K.~J. Obermeyer and Contributors, ``The {VisiLibity} library,''
  \texttt{http://www.VisiLibity.org}, 2008, release 1.

\end{thebibliography}
\end{document}